%% file: main.tex
\documentclass{article}
\usepackage{ijcai19}

\usepackage{times}
\usepackage{soul}
\usepackage{url}
\usepackage[hidelinks]{hyperref}
\usepackage[utf8]{inputenc}
\usepackage[small]{caption}
\usepackage{graphicx}
\usepackage{amsmath}
\usepackage{amsthm}
\usepackage{amssymb}
\usepackage{bbm}
\usepackage{booktabs}
\usepackage{algorithm}
\usepackage{algorithmic}
\usepackage{microtype}
\usepackage{xcolor}
\usepackage{subcaption}

\usepackage{hyperref}

\newtheorem{lemma}{Lemma}
\newtheorem{theorem}{Theorem}
\newtheorem{corollary}{Corollary}

\newtheorem{innercustomthm}{Theorem}
\newenvironment{customthm}[1]
  {\renewcommand\theinnercustomthm{#1}\innercustomthm}
  {\endinnercustomthm}

\newcommand{\todo}[1]{\textcolor{red}{TODO: {#1}}}

\newcommand{\sat}{s}
\newcommand{\npe}{p}
\newcommand{\choc}{\textit{Choc}}
\newcommand{\kale}{\textit{Kale}}

\newcommand{\calA}{\mathcal{A}}
\newcommand{\calB}{\mathcal{B}}
\newcommand{\calH}{\mathcal{H}}

\newcommand{\calM}{\mathcal{M}}

\newcommand{\calS}{\mathcal{S}}
\newcommand{\calZ}{\mathcal{Z}}
\newcommand{\frakb}{\mathfrak{b}}

\newcommand{\veps}{\varepsilon}

\newcommand{\scite}{\shortcite}

\DeclareMathOperator*{\E}{\mathbb{E}}
\DeclareMathOperator*{\argmin}{arg\,min}
\DeclareMathOperator*{\argmax}{arg\,max}

\newcommand{\ignore}[1]{}

\title{Advantage Amplification in Slowly Evolving Latent-State Environments}

\author{
Martin Mladenov$^1$\and
Ofer Meshi$^1$\and
Jayden Ooi$^1$\and
Dale Schuurmans$^{1,2}$\and
Craig Boutilier$^1$
\affiliations
$^1$Google AI, Mountain View, CA, USA\\
$^2$Department of Computer Science, University of Alberta, Edmonton, AB, Canada
\emails
\{mmldanenov,meshi,jayden,schuurmans,cboutilier\}@google.com
}

\begin{document}

\maketitle

\begin{abstract}
\begin{small}
Latent-state environments with long horizons, such as those faced by recommender systems, pose
significant challenges for reinforcement learning (RL). In this work, we identify
and analyze several key hurdles for RL in such environments, including belief state
error and small action advantage. We develop a general principle called \emph{advantage
amplification} that can overcome these hurdles through the use of temporal
abstraction. We propose several aggregation methods and prove they
induce amplification in certain settings. We also bound the loss
in optimality 
incurred by our methods in
environments where latent state evolves slowly and demonstrate their performance
empirically in a stylized user-modeling task.
\end{small}
\end{abstract}

\section{Introduction}
\label{sec:intro}

\emph{Long-term value (LTV)} estimation and optimization
is of increasing
importance in the design of \emph{recommender systems (RSs)}, and other
user-facing 
systems.
Often the problem is framed as a \emph{Markov decision process (MDP)} and solved using
MDP algorithms or \emph{reinforcement learning (RL)}
\cite{shani:jmlr05,taghipour:recsys07,choi2018reinforcement,zhao2017deep,mirrokni:wine12,logisticMDPs:ijcai17}.
Typically, actions are the set of recommendable
items\footnote{Item \emph{slates}
are often recommended, but we ignore this here.}; states reflect information
about the user (e.g., static attributes, past interactions, context/query); and rewards measure some form of user engagement (e.g., clicks, views, time spent,
purchase). Such \emph{event-level models}
have seen some success, but current
state-of-the-art is limited to very short horizons.

When dealing with long-term user behavior, it is vital
to consider the impact of recommendations on user \emph{latent state}
(e.g., satisfaction, latent interests, or item awareness)
which often
governs both immediate and long-term behavior. Indeed, the main promise of using
RL/MDP models for RSs is to (a) identify latent state (e.g., uncover topic interests via exploration) and (b) influence the latent state (e.g., create new interests or improve awareness and satisfaction). That said, evidence is emerging that at least some aspects of user latent state \emph{evolve very slowly}. For example, Hohnhold et al.~\scite{hohnhold:kdd15} show that varying ad quality and ad load induces slow, but inexorable (positive or negative) changes in user click propensity over a period of months, while Wilhelm et al.~\scite{wilhelm:cikm18} show that explicitly diversifying recommendations in YouTube induces similarly slow, persistent changes in
user engagement (see such slow ``user learning'' curves in Fig.~\ref{fig:slowcurves}).


\begin{figure*}[t]
\centering
\begin{subfigure}{.5\textwidth}
  \centering
  \includegraphics[width=.7\textwidth]{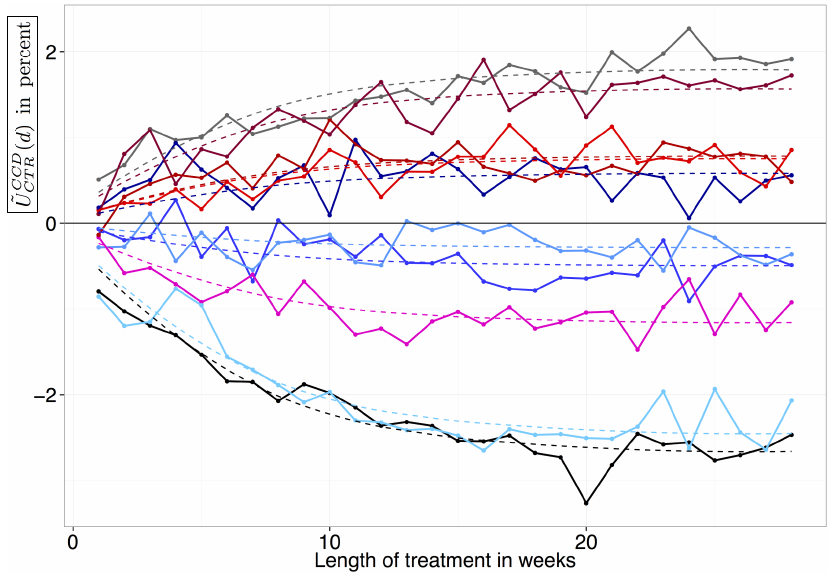}
\end{subfigure}%
\begin{subfigure}{.5\textwidth}
  \centering
  \includegraphics[width=.6\textwidth]{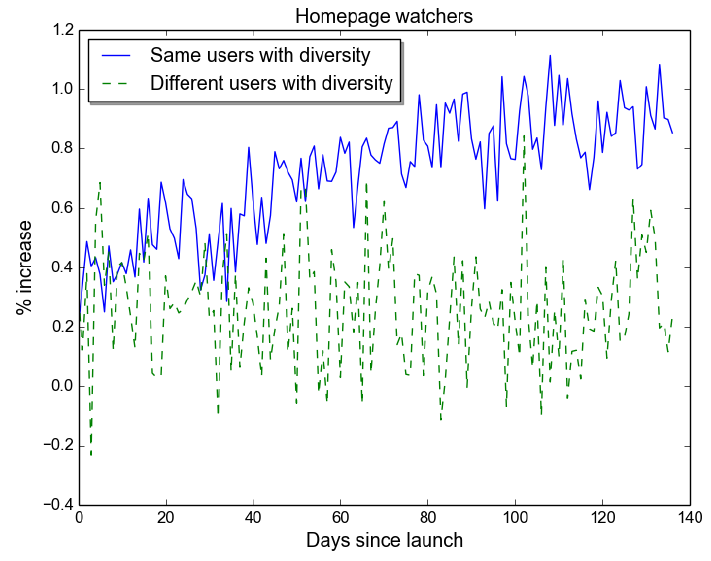}
\end{subfigure}
\vspace*{-2mm}
{\small \caption{Gradual user response: (a) ad load/quality \protect\cite{hohnhold:kdd15}; (b) YouTube recommendation diversity \protect\cite{wilhelm:cikm18}. \label{fig:slowcurves}}}
\vspace*{-3mm}
\end{figure*}


Event-level RL in such settings is challenging for several reasons. First, the effective horizon over which an
RS policy influences the latent state can extend up to 
$O(10^4\textup{--}10^5)$ state transitions. Indeed, the cumulative effect of recommendations is
vital for LTV optimization, but the long-term impact of any \emph{single}
recommendation is often dwarfed by immediate reward differences.
Second, the MDP is partially observable, requiring some form of belief state estimation. Third, the impact of latent state on immediate
observable behavior is often small 
and very noisy---the problems have a low 
\emph{signal-to-noise ratio (SNR)}.
We detail below how these factors interact.

Given the importance of LTV optimization in RSs, we propose a new technique called
\emph{advantage amplification} to overcome these challenges. Intuitively,
amplification seeks to overcome the error induced by state estimation
by introducing
(explicit or implicit) \emph{temporal abstraction} across policy space.
We require that policies take a series of actions, thus allowing more accurate value estimation by mitigating the cumulative
effects of state-estimation error. We first consider \emph{temporal aggregation}, where an action is held fixed for a short horizon.
We show that this can lead to significant
amplification of the advantage differences between abstract actions 
(relative to event-level actions).
This is a form of MDP/RL temporal abstraction as used in
\emph{hierarchical RL} \cite{Sutton-etal:AIJ99,barto2003recent} and
can be viewed as options or macros designed for the purpose
of allowing distinction of good and bad behaviors in latent-state domains
with low SNR (rather than, say, for subgoal achievement).
We generalize this by analyzing policies with (artificial) action \emph{switching costs}, which induces similar amplification with more flexibility.

Limiting policies to temporally abstract actions induces potential
sub-optimality \cite{parr:uai98,HMKDB:uai98}. However, since the underlying latent state often evolves slowly
w.r.t.\ the event horizon in RS settings, we identify a ``smoothness'' property 
that is used to bound the induced error of advantage amplification.

Our contributions are as follows. We introduce a stylized model of slow user learning in RSs
(Sec.~\ref{sec:chockale}) and formalize this as a POMDP (Sec.~\ref{sec:background}), defining
several novel concepts, and show how low SNR interacts
poorly with belief-state approximation (Sec.~\ref{sec:problem}). We develop advantage
amplification as a principle 
and prove that \emph{action aggregation}
(Sec.~\ref{sec:aggregation})
and \emph{switching cost regularization}
(Sec.~\ref{sec:switching_cost}) 
provide strong
amplification guarantees with minimal policy loss under suitable conditions. Experiments with
stylized models show the effectiveness of these techniques.\footnote{Proofs, auxiliary lemmas and additional experiments are available in an extended version of the paper.}

\section{User Satisfaction: An Illustrative Example}
\label{sec:chockale}

Before formalizing our problem, we describe
a stylized model reflecting the dynamics of \emph{user satisfaction} as a user interacts with
an RS.
The model is intentionally stylized to help
illustrate key concepts underlying the formal model and analysis developed in the sequel
(hence ignores much of the true complexity of user satisfaction).
Though focused on user-RS engagement, the principles apply more broadly
to any latent-state system with low SNR and slowly evolving latent state.

Our model captures the relationship between a user and an RS
over an extended period (e.g., a content recommender of news, video, or music) through
\emph{overall user satisfaction}, which is not known to the RS.
We hypothesize that satisfaction is one (of several) key latent factors that impacts user engagement; and since new treatments often induce slow-moving or delayed effects on user
behavior, we assume this latent variable evolves slowly as a function of the quality of
the content consumed 
\cite{hohnhold:kdd15} (and see Fig.~\ref{fig:slowcurves} (left)).
Finally, the model captures the tension between (often low-quality) content that encourages short-term
engagement (e.g., manipulative, provocative or distracting content)
at the expense
of long-term engagement; and high-quality content that promotes long-term usage but can sacrifice near-term
engagement.

Our model includes two classes of recommendable items. Some items
induce high immediate engagement, 
but degrade user engagement
over the long run. We dub these ``Chocolate'' ($\choc$)---immediately appealing but not very ``nutritious.''
Other items---dubbed ``Kale,'' less attractive, but more ``nutritious''---induce lower immediate engagement but tend to improve long-term engagement.\footnote{Our model allows
a real-valued continuum of items (e.g., degree of Choc between $[0,1]$ as in our
experiments)
like measures of ad quality. We use the binary form to streamline our
initial exposition.}
We call this the \emph{Choc-Kale model} (CK).
A stationary, stochastic policy can be represented by a single scalar $0\le\pi\le 1$ representing the probability of taking action $\choc$. We sometimes refer to $\choc$  as a ``negative'' and 
$\kale$ as a ``positive'' recommendation.

We use a single latent variable $\sat\in [0,1]$ to capture a user’s
overall satisfaction with the RS. Satisfaction is driven by \emph{net positive exposure} $\npe$, 
which measures total positive-less-negative recommendations, with a discount $0\le\beta<1$ applied to ensure that $\npe$ is bounded:
{\small $\npe \in \left[\frac{-1}{1-\beta},\frac{1}{1-\beta}\right]$}.
We view $\npe$ as a user’s learned perception of the RS and $\sat$ as how this
influences gradual changes in engagement.

A user response to a recommendation $a$ is her \emph{degree of engagement} $g(\sat,a)$, and
depends stochastically on both the quality of the recommendation,
and her latent state $\sat$. 
$g$ is a random function, e.g., responses might be
normally distributed:
{\small $g(\sat,a) \sim N(\sat\cdot\mu_a, \sigma_a^2) ~~\text{for } a\in\{\choc,\kale\}$}.
We use $g(\sat, a)$ also to denote expected engagement.
We require that $\choc$ results in greater immediate (expected) engagement than $\kale$, $g(\sat, \kale) < g(\sat, \choc)$, for any fixed $\sat$. 


\begin{figure*}[ht]
    \centering
    \begin{subfigure}[b]{0.24\textwidth}
        \includegraphics[width=\textwidth]{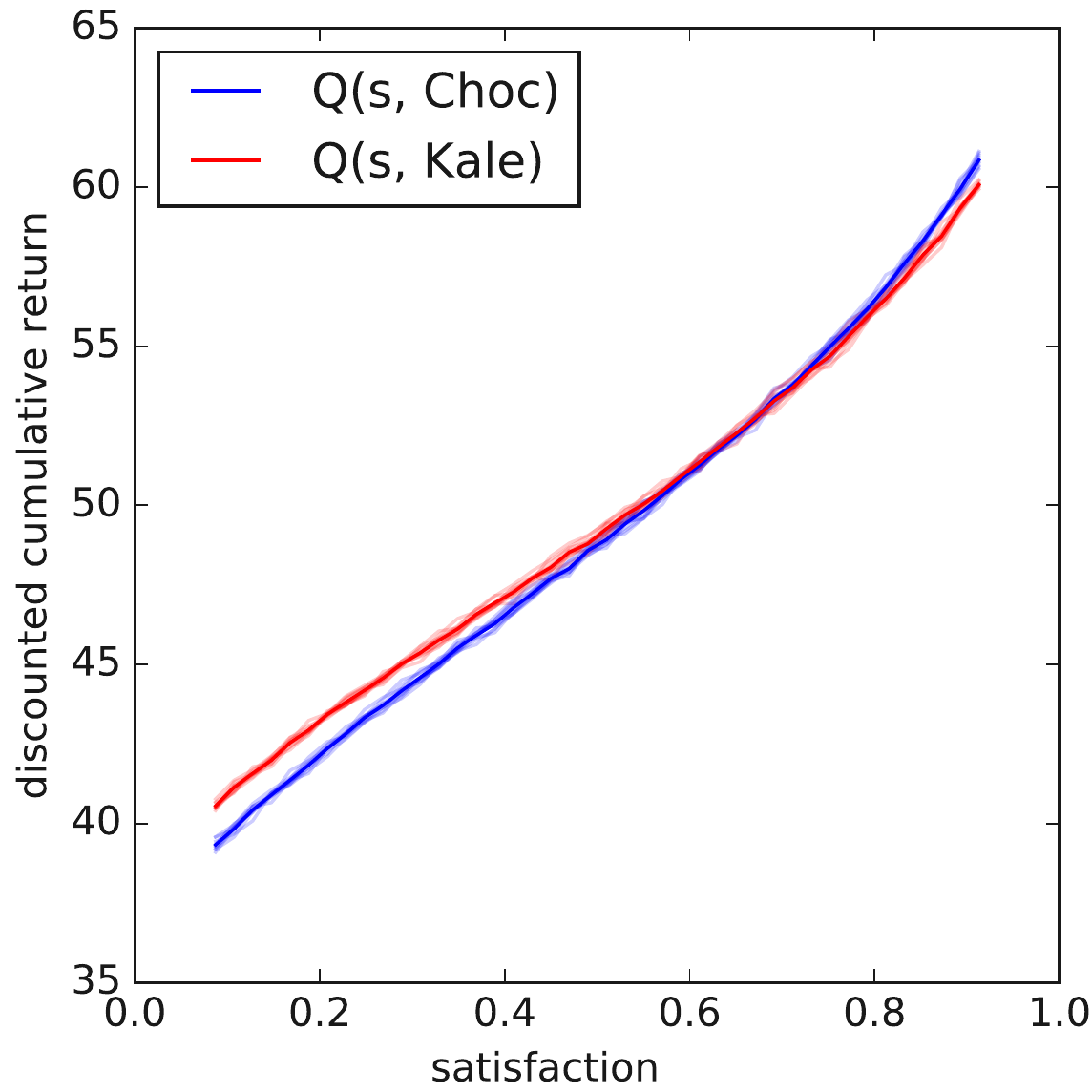}
        \caption{event-level, fully observable}
        \label{fig:qvalues:no_agg_low_noise}
    \end{subfigure}
    ~ 
    \begin{subfigure}[b]{0.24\textwidth}
        \includegraphics[width=0.985\textwidth]{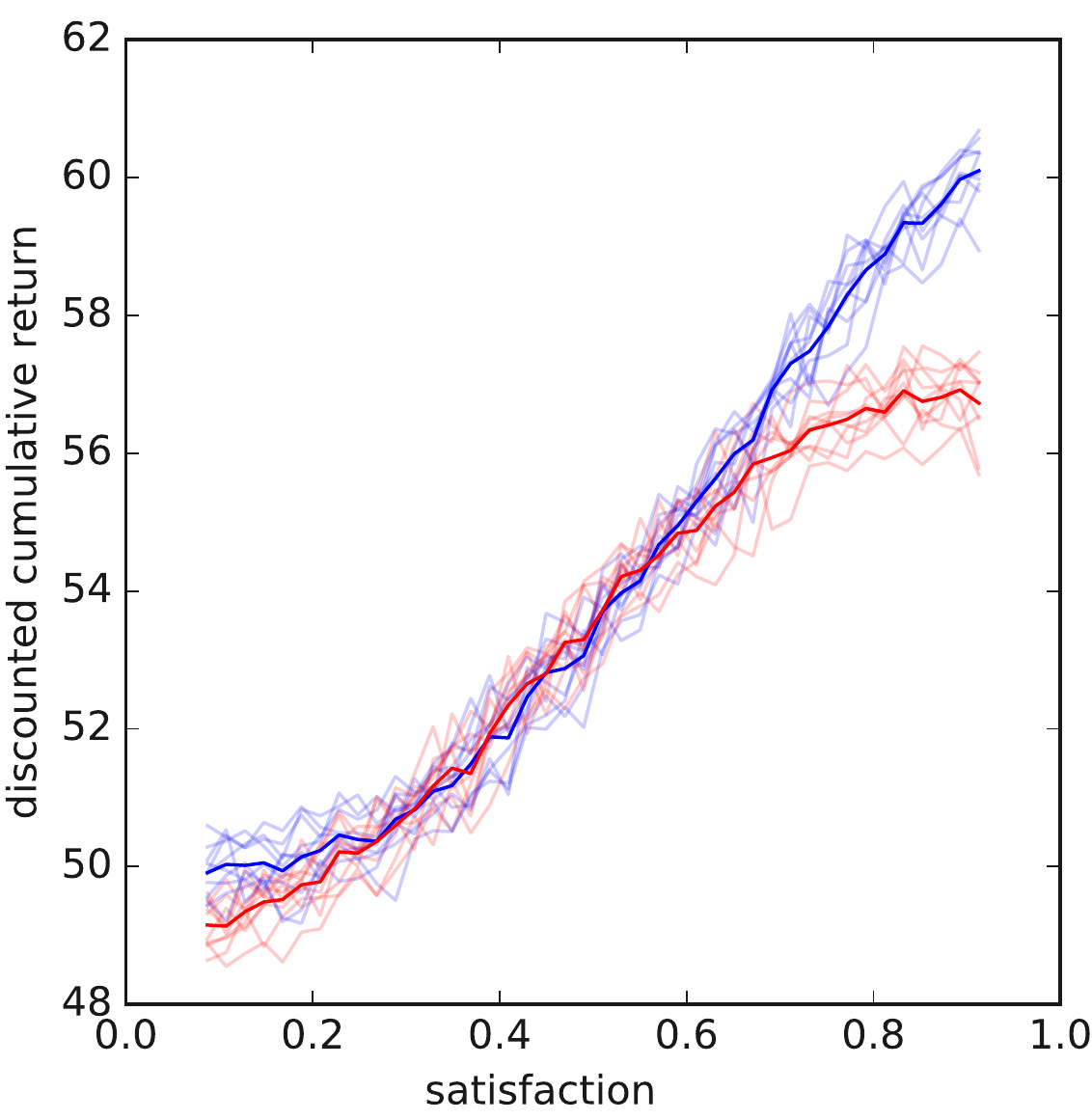}
        \caption{event-level, noisy}
        \label{fig:qvalues:no_agg_w_noise}
    \end{subfigure}
        \begin{subfigure}[b]{0.24\textwidth}
        \includegraphics[width=\textwidth]{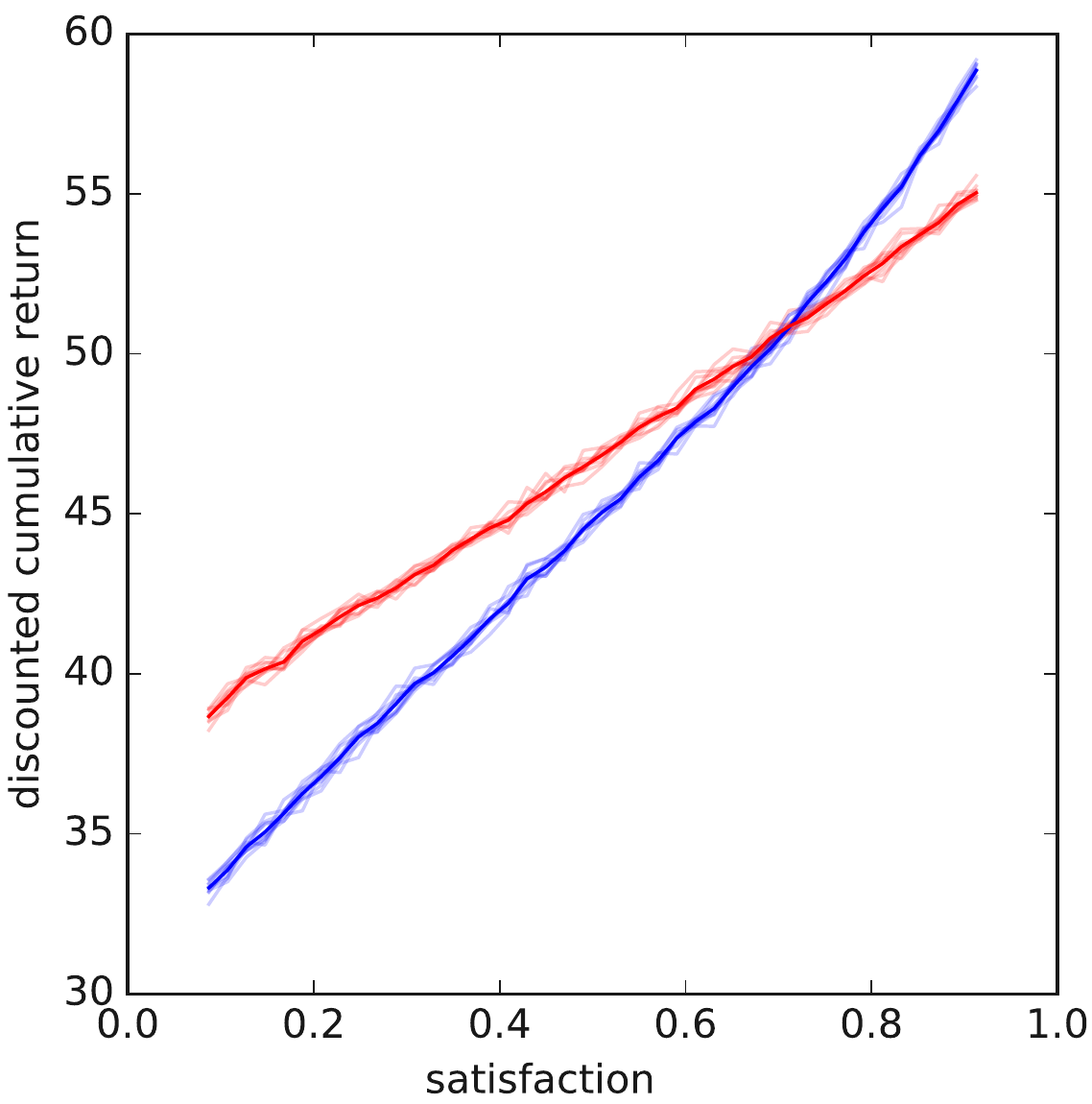}
        \caption{aggregated, fully observable}
        \label{fig:qvalues:agg_low_noise}
    \end{subfigure}
    ~ 
    \begin{subfigure}[b]{0.24\textwidth}
        \includegraphics[width=\textwidth]{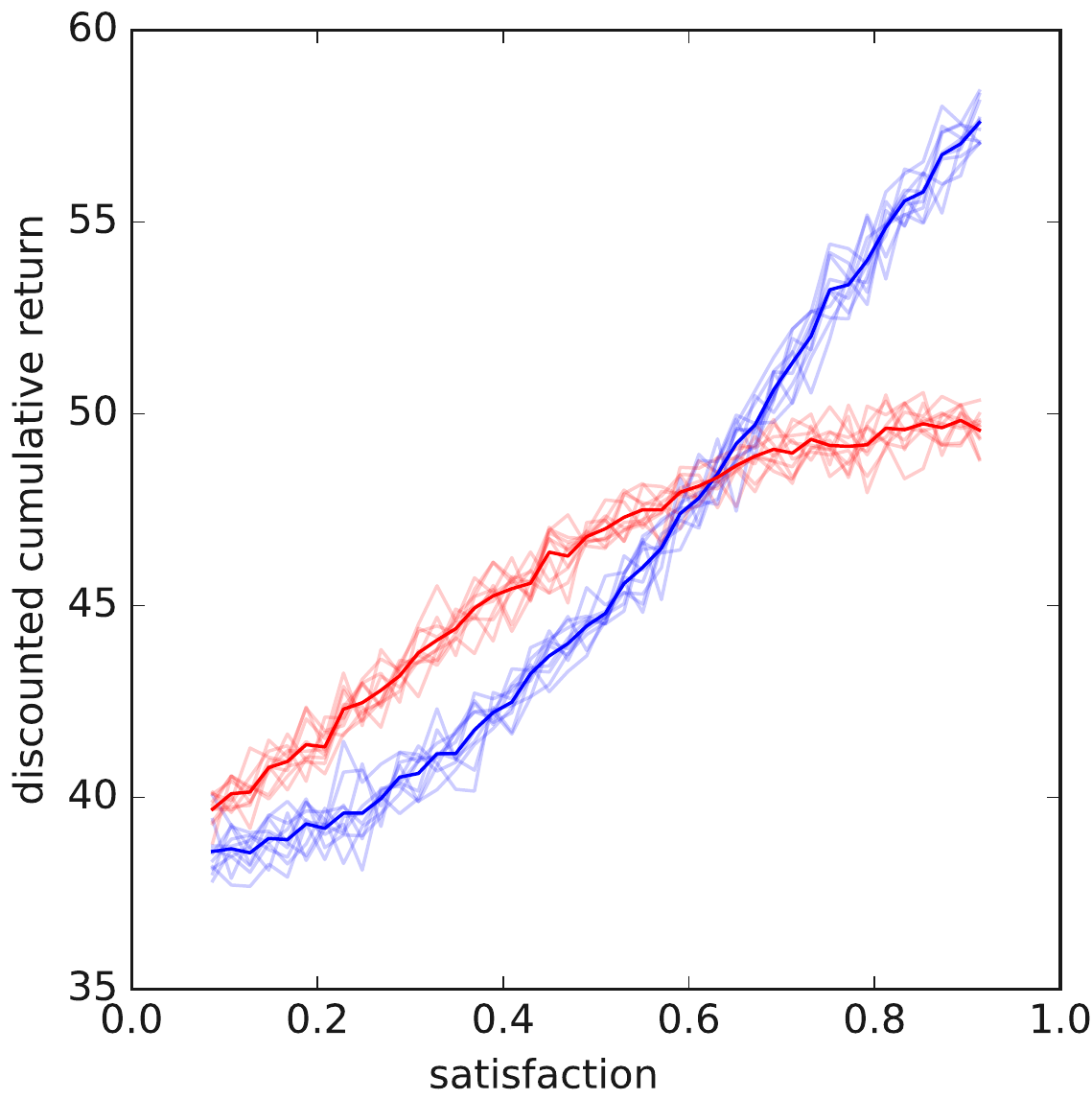}
        \caption{aggregated, noisy}
        \label{fig:qvalues:agg_w_noise}
    \end{subfigure}
      \vspace*{-3mm}
      \caption{Q-values per satisfaction level in the Choc-Kale model.}
      \label{fig:qvalues}
      \vspace*{-3mm}
\end{figure*}

The dynamics of $\npe$ is straightforward. A $\kale$ exposure increases $\npe$ by 1
and $\choc$ decreases it by 1 (with discounting):
$\npe_{t+1} \leftarrow \beta\npe_t + 1$ with $\kale$ (and $-1$ with $\choc$).
Satisfaction 
$\sat$ is a user-learned function of $\npe$ and follows a sigmoidal learning curve:
$\sat(\npe) = 1/(1+e^{-\tau\npe})$,
where $\tau$ is a temperature/learning rate parameter.
Other learning curves are possible, but the sigmoidal model captures both positive and negative exponential learning 
as hypothesized in psychological-learning literature \cite{thurstone:Learning1919,jaber:HumanLearning2006}
and as observed in the empirical curves in Fig.~\ref{fig:slowcurves}.\footnote{Such learning
curves are often reflective of aggregate behavior, obscuring individual differences that are much less
``smooth.'' However, unless cues are available that allow us to model such individual differences, the
aggregate model serves a valuable role even when optimizing for individual users.}

We compute the Q-values of $\choc$ and $\kale$ for each satisfaction level $s$ and plot them in Fig.~\ref{fig:qvalues:no_agg_low_noise}. We observe that when satisfaction is low, $\kale$ is a better
recommendation, and above some level $\choc$ becomes preferable, as expected. We also see that for any $s$ the difference in Q-values is rather small.
With additional noise, the Q-values become practically indistinguishable for a large range of satisfaction levels (Fig.~\ref{fig:qvalues:no_agg_w_noise}), which illustrates the hardness of RL in this setting.




\section{Problem Statement}
\label{sec:background}

We outline a basic latent-state control problem as a partially observable MDP (POMDP) that encompasses the notions above. We highlight several 
properties that play a key role in the analysis of latent-state RL we develop in
the next section.

We consider environments that can be modeled as a POMDP
$\calM = \langle \calS, \calA, T, \calZ, O, R, \frakb_0, \gamma \rangle$ \cite{smallwood-sondik:or73}. States $\calS$ reflect user latent state
and other observable aspects of the domain: in the CK model, this is simply a user's
current satisfaction $\sat$. Actions $\calA$
are recommendable items:
in CK, we distinguish only $\choc$ from $\kale$.
The transition kernel $T(s,a,s')$
in the CK model is $T(\sat',a,\sat) = 1$ 
if {\small $s' = (1+\exp(\beta \log\left(1-1/s\right)-\beta\tau a))^{-1}$}, where
$a$ is 1 (resp., -1) for action $\kale$ (resp., $\choc$).\footnote{This is easily
randomized if desired.}
Observations $\calZ$ reflect observable user behavior and
$O(s,a,z)$ the probability
of $z\in\calZ$ when $a$ is taken at state $s$.
In CK, $\calZ$ is the observed engagement with a recommendation
while $O$ reflects the
random realization of $g(\sat,a)$. The immediate reward $R(s,a)$ is (expected) user
engagement (we let $r_{\max} = \max_{s,a} R(s,a)$),
$\frakb_0$ the initial state distribution, and $\gamma\in[0,1)$ the discount factor.

In this POMDP, an RS does not have access to the true state $s$, but must 
generate policies that depend only on the sequence of past
action-observation pairs---let
${\cal H}^\ast$ be the set of all finite such sequences
$(a_t, z_t)_{t\in \mathbb{N}}$. Any such \emph{history} can be summarized, via optimal Bayes
filtering, as a distribution or \emph{belief state} $\mathfrak{b}\in\Delta(\calS)$.
More generally, this ``belief state'' can be 
\emph{any summarization} of ${\calH}^\ast$ used to make decisions.
It may be, say, a
collection of sufficient statistics, or a deep recurrent embedding of history.
Let $\calB$ denote the set of (realizable) belief states.
We also require a mapping
$U:\calB \times \calA \times \calZ \rightarrow \calB$ that describes the update $U(\mathfrak{b}, a,z)$ 
of any $\mathfrak{b}\in\calB$ given $a\in\calA, z\in\calZ$. The pair $(\calB,U)$ defines our \emph{representation}.

A \emph{(stochastic) policy} is a mapping $\pi:\calB\rightarrow\Delta(\calA)$ that selects an
action distribution $\pi(\mathfrak{b})$ for execution given belief $\mathfrak{b}$; we write
$\pi(a|\mathfrak{b})$ to indicate the probability of action $a$. Deterministic policies are
defined in the usual way. The \emph{value} of a policy $\pi$ is given by the standard
recurrence:\footnote{Here $R(\mathfrak{b}, a)$ and $\Pr(z|\mathfrak{b}, a)$ are given by expectations of $R$ and $O$, respectively,
w.r.t.\ $s\sim \mathfrak{b}$ if $\mathfrak{b}\in\Delta(\calS)$. The interpretation for other
representations is discussed below.}
\begin{small}
\begin{equation}
\label{eq:vf}
V^\pi(\mathfrak{b})\! =\! \E_{a\sim\pi(\mathfrak{b})} \left[ R(\mathfrak{b}, a)\! +\! \gamma\sum_{z\in\calZ}\! \Pr(z|\mathfrak{b},a) V^\pi(U(\mathfrak{b}, a,z)) \right]
\end{equation}
\end{small}
We define $Q^\pi(\mathfrak{b}, a)$ by fixing $a$ in Eq.~\ref{eq:vf} (rather than taking an
expectation). An \emph{optimal policy} $\pi^\ast = \sup V^\pi$ over $\calB$
has value (resp., Q) function $V^\ast$ (resp., $Q^\ast$). Optimal
policies and values can be computed using 
dynamic programming
or learned using (partially observable) RL methods.
When we learn a Q-function $Q$, whether exactly or approximately, the
\emph{policy induced by $Q$} is the greedy policy $\pi(\frakb) = \arg\max_a Q(\frakb,a)$ and its \emph{induced value function} is
$V(\frakb) = \max_a Q(\frakb,a) = Q(\frakb,a^\ast(\frakb))$.
The \emph{advantage function}
$A(a,\frakb) = V^\ast(\mathfrak{b}) - Q^\ast(\mathfrak{b}, a)$ reflects the difference in
the expected value of doing $a$ at $\mathfrak{b}$ (and then acting optimally) vs.\ 
acting optimally at $\mathfrak{b}$ \cite{baird:phd1999}.
If $a_2$ is the second-best action at $\frakb$, the \emph{advantage} of that belief state is
$A(\mathfrak{b}) =  V^{\ast}(\mathfrak{b}) - Q^{\ast}(\mathfrak{b}, a_2)$.

Eq.~\ref{eq:vf} assumes optimal Bayesian filtering, i.e., the
representation $(\calB,U)$ must be such that the (implicit) expectations over $R$ and $O$
are exact for any history that maps to 
$\mathfrak{b}$.
Unfortunately, exact recursive state estimation is intractable,
except for special cases (e.g., linear-Gaussian control). As a consequence,
\emph{approximation schemes} are used in practice (e.g., variational projections \cite{boyen:uai98};
fixed-length histories, incl.\ treating observations
as state \cite{singh:icml94};
learned PSRs \cite{littman_psrs:nips02}; 
recursive policy/Q-function representations \cite{recurrent_psr}).
Approximate
histories render the
process non-Markovian;
as such, a counterfactually estimated Q-value of a policy
(e.g., using offline data) differs from its 
\emph{true} value due to modified latent-state dynamics
(not reflected in the data). In this case, any RL method that treats
$\mathfrak{b}$ as (Markovian) state induces a suboptimal policy.
We can bound the induced suboptimality using
\emph{$\veps$-sufficient statistics} \cite{epsilon_sufficient}.
A function $\phi: {\cal H}^\ast \rightarrow \calB$ 
is an $\veps$-sufficient statistic if, for all 
$H_t \in {\cal H}^\ast$,
\[ \left|p(s_{t+1}|H_t) - p(s_{t+1}|\phi(H_t)) \right|_\mathrm{TV} < \veps~, \]
where $|\cdot|_{\mathrm{TV}}$ is the total variation distance.
If $\phi$ is $\veps$-sufficient, then any MDP/RL
algorithm that constructs an ``optimal''
value function $\hat{V}$ over $\calB$ incurs a bounded loss w.r.t.\ $V^\ast$ \cite{epsilon_sufficient}:
\begin{small}
\begin{equation}
\label{eq:non_markov_loss}
\left|V^{\ast}(\phi(H)) - {\hat V}(\phi(H))\right| \leq \frac{2\veps r_\text{max}}{(1-\gamma)^3}.
\end{equation}
\end{small}


The errors in Q-value estimation induced by limitations of $\calB$ are irresolvable (i.e., they are a form of model \emph{bias}),
in contrast to error induced by limited data. Moreover, any RL method relying only on offline data is subject to the above bound, regardless of whether the Q-values are estimated directly or not.
The impact of this error on model performance can be
related to certain properties of the underlying domain as we outline below. A useful
quantity for this purpose is the \emph{signal-to-noise ratio (SNR)} of a POMDP, defined as:
\[\mathfrak{S}\triangleq \frac{\sup_{\mathfrak{b}} A(\mathfrak{b})}{\sup_{\mathfrak{b}: A(\mathfrak{b}) \leq 2\veps r_\mathrm{max}/(1-\gamma)^2} A(\mathfrak{b})} - 1,\]
(the denominator is treated as $0$ if no $\mathfrak{b}$ meets the condition).

As discussed above, many aspects of user latent state, such as satisfaction, evolve slowly.
We say a POMDP is 
\emph{$L$-smooth} 
if, for all $\mathfrak{b}, \mathfrak{b}^\prime \in \calB$, and $a\in A$
s.t.\ $T(\mathfrak{b}',a,\mathfrak{b}) > 0$, we have  
\begin{small}
\[|Q^{\ast}(\mathfrak{b}, a)-Q^{\ast}(\mathfrak{b}',a)| \leq L.\] 
\end{small}%
Smoothness ensures that for any state reachable under an action $a$, the optimal
Q-value of $a$ does not change much.



\section{Advantage Amplification}
\label{sec:advantage}

We now detail how low SNR causes difficulty for RL in POMDPs,
especially
with long horizons (Sec.~\ref{sec:problem}). We
introduce the principle of \emph{advantage amplification} to address it (Sec.~\ref{sec:aggregation})
and describe two realizations, temporal aggregation (Sec.~\ref{sec:aggregation}) and switching cost (Sec.~\ref{sec:switching_cost}).

\subsection{The Impact of Low SNR on RL}
\label{sec:problem}
The bound Eq.~\eqref{eq:non_markov_loss} can help
assess the impact of low SNR on RL.
Assume that
policies, values or Q-functions are learned using an approximate belief
representation $(\calB,U)$
that is
$\veps$-sufficient. 
We first show that the error induced by $(\calB,U)$
is tightly coupled to optimal action advantages in the domain. 

Consider an RL agent that learns Q-values using a behavior (data-generating) policy $\rho$. The
non-Markovian nature of $(\calB,U)$ means that: (a) the resulting estimated-optimal policy $\pi$ will have \emph{estimated} values $\hat{Q}^{\pi}$ that differ from its \emph{true} values $Q^{\pi}$; and (b) the estimates $\hat{Q}^{\pi}$ (hence, the choice of $\pi$ itself) will depend on $\rho$.
We bound the loss of $\pi$ w.r.t.\ the
optimal $\pi^{\ast}$ (with exact filtering) as follows. First,
for any (belief) state-action pair $(\mathfrak{b}, a)$, suppose the maximum difference between its
inferred and optimal Q-values is bounded for
any $\rho$: {\small $|Q^{\ast}(\mathfrak{b}, a)-Q^{\pi}(\mathfrak{b}, a)| \leq \delta$}.
By Eq.~\eqref{eq:non_markov_loss} we
set 
\begin{small}
\begin{equation}
\delta = \frac{\veps Q_{\max}}{1-\gamma} \leq \frac{\veps r_{\max}}{(1-\gamma)^2}. \label{eq:setdelta}
\end{equation}
\end{small}


If $\mathfrak{b}$ has small advantage $A(\mathfrak{b}) \leq 2\delta$,
under behavior policy $\rho$, the estimate $\hat{Q}(\mathfrak{b}, a_2)$ (the second-best action)
can exceed that of $\hat{Q}(\mathfrak{b}, a^\ast(\frakb))$;
hence $\pi$ executes $a_2$. If $\pi$ visits $\mathfrak{b}$ (or states
with similarly small advantages) at a constant rate, the loss w.r.t.\ $\pi^\ast$ compounds, inducing $O(\frac{2\delta}{1-\gamma})$ error.

The tightness of the second part of the argument depends on the structure of
the advantage function $A(\mathfrak{b})$. To illustrate, consider two extreme regimes. First, if $A(\mathfrak{b})\geq 2\delta$ 
at all $\mathfrak{b}\in\calB$,
i.e., if SNR $\mathfrak{S} = \infty$,
state
estimation error has no impact on the recovered policy and incurs no loss.
In the second regime, if all $A(\mathfrak{b})$ are less than (but on the order of) $2\delta$,
i.e., if $\mathfrak{S} = 0$, 
then the inequality is tight
provided $\rho$ saturates the state-action error bound.
We will see below that low-SNR environments with long horizons (e.g., practical RSs, the stylized CK model) often 
have such small (but non-trivial) advantages across a wide
range of state space.

The latter situation in illustrated on Fig.~\ref{fig:qvalues}. In Fig.~\ref{fig:qvalues:no_agg_low_noise}, the Q-values of the CK model are plotted against the level of satisfaction (as if satisfaction were fully observable). The small advantages are notable. Fig.~\ref{fig:qvalues:no_agg_w_noise} shows the Q-value estimates for $10$ independent tabular Q-learning reruns (the thin lines show the individual runs, the thick lines show the average) where noise is added to $s$. The corrupted $Q$-values at all but the highest satisfaction levels are essentially indistinguishable, leading to extremely poor policies.

\subsection{Temporal Abstraction: Action Aggregation}
\label{sec:aggregation}

There is a third regime in which state error is relatively benign. 
Suppose the advantage at each state $\mathfrak{b}$ is either small, 
$A(\mathfrak{b})\leq\sigma$,
or large,
$A(\mathfrak{b})>\Sigma$
for some constants $\sigma \ll 2\delta \leq \Sigma$.
The induced policy incurs a loss of $\sigma$ at small-advantage states, and no loss on states with large advantages. 
This leads to a compounded 
loss of at most $\frac{\sigma}{1-\gamma}$, 
which may be much smaller than the $\frac{\veps r_\mathrm{max}}{(1-\gamma)^2}$ error (Eq.~\ref{eq:setdelta})
depending on $\sigma$.

If the environment is smooth, \emph{action aggregation} can be used to restructure a problem falling in the second regime
to one in this third regime, with $\sigma$ depending on the level of smoothness. This can
significantly reduce the impact of estimation error on policy quality by turning the problem into one that
is essentially Markovian. More specifically, if at state $\mathfrak{b}$, we know that the optimal (stationary) policy takes
action $a$ for the next $k$ decision periods, we consider a reparameterization ${\cal M}^{\times k}$ of the belief-state
MDP where, at $\mathfrak{b}$, all actions are executed $k$ times in a row, no matter what the subsequent $k$ states are.
In this new problem, the Q-value of the optimal repeated action $Q^{\ast}(\mathfrak{b},a^{\times k})$ 
is the same as that of its event-level counterpart $Q^{\ast}(\mathfrak{b}, a)$,
since the same sequence of expected rewards will be generated.
Conversely, all suboptimal actions incur a cumulative \emph{reduction} in Q-value in ${\cal M}^{\times k}$ 
since their suboptimality \emph{compounds} over $k$ periods. Thus, in ${\cal M}^{\times k}$, the optimal policy 
$\pi^{\times k \ast}$
generates the same cumulative discounted return as the event-level optimal policy,
while the advantage of $a^{\times k}$ over any other repeated action $a'^{\times k}$ at $\mathfrak{b}$ is larger than that of $a$ over $a'$ in the event-level problem.

To derive bounds, note that,
for an $L$-smooth POMDP,
at any state where the advantage is at least $2kL$, the optimal action persists for the 
next $k$ periods (its Q-value can decrease by at most $L$ while that of the second-best can at most increase by $L$). 
If we apply aggregation only at such states, the advantage increases to some value $\Sigma$, putting us
in regime 3 (i.e., the advantage is either less than $\sigma = 2kL$ or more than $\Sigma$).
Of course, we cannot ``cherry-pick'' only states with high advantage for aggregation; but aggregating
over all states induces some loss due to the inability to switch actions quickly. We bound that cost
in computing $\sigma$ and $\Sigma$.
This allows us to first lower bound the regret of the best $k$-aggregate policy:
\begin{theorem}
\label{thm:approxbound}
Let $k$ be a fixed horizon, and let $Q^\ast$---the event-level optimal $Q$ function---be $L$-smooth. Then for all $\mathfrak{b}$, $|V^\ast(\mathfrak{b}) - V^{\times k\ast}(\mathfrak{b})| \leq \frac{2kL}{1-\gamma}$, where $ V^{\times k\ast}(\mathfrak{b})$ is the value of state $\mathfrak{b}$ under an optimal $k$-aggregate policy.\footnote{The reparameterized problem 
is also an MDP, so 
the optimal value function 
and deterministic policy are well-defined. }
\end{theorem}
This theorem is proved by constructing a policy which switches actions every $k$ events and showing that it has bounded regret. This policy, at the start of any $k$-event period, adopts the optimal action from the unaggregated MDP at the initiating state.
Due to smoothness, Q-values cannot drift by more than $kL$ during the period, after which the policy corrects itself.
This, together with the reasoning abouve, offers an amplification guarantee:
\begin{theorem}
\label{thm:amplification_agg}
In an $L$-smooth MDP, let $k$ be a fixed repetition horizon. For any belief state where
$A(\mathfrak{b})\ge 2kL$,
the $k$-aggregate-horizon advantage is bounded below:
\begin{small}
\begin{align*}
Q^{\times k\ast}&(\mathfrak{b}, a^{\times k}) - Q^{\times k\ast}(\mathfrak{b}, a'^{\times k}) \\
\geq&~ A(\mathfrak{b})\frac{1-\gamma^{k}}{1-\gamma} - 2L\frac{\gamma - (1+k - \gamma k)\gamma^{k+1}}{(1-\gamma)^2} - \frac{2kL}{1-\gamma}.
\end{align*}
\end{small}
\end{theorem}
This result is especially useful when the event-level advantage is more than 
$\sigma = \frac{2kL}{1-\gamma}$.
In this case, an aggregation horizon of $k$ can mitigate the 
adverse effects of approximating belief state with an $\veps$-sufficient statistic for 
for an $\veps$ up to:
\begin{small}
\[\veps_\mathrm{max} \leq L\frac{k(\gamma-\gamma^k) - \gamma(1 - (1+k - \gamma k)\gamma^{k})}{r_\mathrm{max}}\]
\end{small}
at the cost of the aggregation loss of $\frac{2kL}{1-\gamma}$. 

Figs.~\ref{fig:qvalues:agg_low_noise} and ~\ref{fig:qvalues:agg_w_noise} illustrate the benefit of action aggregation: they show the Q-values of the $k$-aggregated CK model with $k=5$ with both perfect and imperfect state estimation, respectively
(the amount of noise is the same as inFig.~\ref{fig:qvalues:no_agg_w_noise}). As we show Sec.~\ref{sec:empirical}, 
the recovered policies incur very little loss due to state estimation error. 
\ignore{To illustrate the value of noise rejection, consider Fig.~\ref{fig:ideal_noise_rejection}, which plots the hypothetical worst-case loss with $\veps_\mathrm{max}$ as above for $\gamma=0.99, L=0.1, r_\mathrm{max}=1$ versus the post-aggregation $\frac{2kL}{1-\gamma}$ against different levels of $k$. 
}

We conclude with the following observation.
\begin{corollary}
Optimal repeating policies are near-optimal for the event-level problem as $L\rightarrow 0$ and amplification at every state is guaranteed.
\end{corollary}

\ignore{
If we further assume that $|Q(a,s) - Q(b,s)| \geq \eta$, we can show that the certain levels of state estimation error can be completely rejected. 

\begin{theorem}
Let $\hat{s}$ be an $\veps$-sufficient statistic and  $|Q(a,s) - Q(b,s)| \geq \eta$. Then, the state estimation error will have no effect on the optimal value function if $L\geq ...$
\end{theorem}
\begin{proof}
TODO
\end{proof}

}

\ignore{
\begin{itemize}
    \item Still missing: a discussion of how aggregation might improve ability to improve state estimation
    \item discuss constant-action aggregation as general illustration of the core principle of advantage amplification. Makes a strong assumption about the domain, but illustrates the main points
    \item spell out constant-action (horizon k) results: key lemmas leading to main theorem on advantage amplification for constant actions; and results relating to conditions under which this will be optimal (probably state specific). Need to relate very specifically to degree of state estimation error.
    \item discuss methods for finding the optimal horizon k
    \item relate to options/macros and hierarchical RL, but emphasize a different motivation for constructing macros
    \item Can we give a more general treatment: design a set of macros/options that serve as a reasonable basis w.r.t.\ advantage amplification?
    \item Give some empirical results (or defer to separate section)
\end{itemize}
} 

\subsection{Temporal Regularization: Switching Cost}
\label{sec:switching_cost}

As discussed above, temporal aggregation is guaranteed to improve learning in 
slow environments. It has, however, certain practical drawbacks due to its inflexibility. 
One such drawback is that, in the non-Markovian setting induced by belief-state approximation,
training data should ideally be collected using a $k$-aggregated
behavior policy.\footnote{This is unnecessary if the system is Markovian,
since $(s,a,r,s')$ tuples may be reordered to emulate any behavioral policy.}
Another drawback arises if the $L$-smoothness assumption is partially violated. For example,
if certain rare events 
cause large changes
in state or reward for short periods, the changes in Q-values may be abrupt, but are harmless from an SNR
perspective if they induce large advantage gaps. An agent ``committed'' to a constant action during
an aggregation period is unable to react to such events. We thus propose a more flexible
advantage amplification mechanism, namely, a \emph{switching-cost regularizer}. 
Intuitively, instead of fixing an aggregation horizon, we impose a fictitious cost (or penalty)
$T$
on the agent whenever it changes its action.

More formally, the goal in the \emph{switching-cost (belief-state) MDP} is to find an optimal
policy defined as:
\begin{small}
\begin{equation}
\label{eq:temp_reg}
\pi^{\ast} = \argmax_{\pi} \sum_t \gamma^t \mathbb{E}_{\pi}\left(R_{t} - T \cdot {\mathbbm 1}[a_t \neq a_{t-1}] \right).
\end{equation}
\end{small}%
This problem is Markovian in the extended state space $\calB\times A$ representing the current
(belief) state and the previously executed action. This state space allows the switching 
penalty to be incorporated into the reward function as 
{\small $R(\mathfrak{b}, a_{t-1}, a_t) = R(\mathfrak{b}, a_{t}) -T \cdot {\mathbbm 1}[a_t \neq a_{t-1}]$}.

The switching cost induces an implicit \emph{adaptive action aggregation}---after executing action 
$a$, the agent will keep repeating $a$ until the cumulative advantage of switching to a 
different action exceeds the switching cost $T$. We can use this insight to bound the maximum 
regret of such a policy (relative to the optimal event-level policy) and also provide
an amplification guarantee, as in the case with action aggregation.

In the case of problems with $2$ actions, we can analyze the action of the switching cost regularizer in a relatively intuitive way. As with Thm.~\ref{thm:approxbound}, we bound the regret induced by the switching cost by 
constructing a policy that behaves as if it had to pay $T$ with every action switch. In 
particular, the optimal policy under this penalty
adopts the action of the event-level optimal policy at some state $\mathfrak{b}_t$, then holds it fixed until its expected regret 
for not switching to a \emph{different} action dictated by the event-level optimal policy exceeds
$T$. Suppose the time at which this switch occurs is ($t+\omega$). The regret of this agent is no 
more than the regret of an agent with the option of paying $T$ upfront in order to follow the event-level optimal policy for $\omega$ steps. We can show that the same regret bound holds if the 
agent were paying to switch to the best fixed action for $\omega$ steps instead of following the
event-level optimal policy. This allows derivation of the following bound:
\begin{theorem}
\label{thm:regret_switch}
The regret of the optimal switching cost policy for a $2$-action MDP is less than $\frac{2\kappa L}{1-\gamma}$, where
\begin{small}
\[\kappa = \frac{\log\gamma + (\gamma - 1)W\left(\frac{\gamma^{1/(1-\gamma)}}{\gamma-1}\left(\frac{(1-\gamma)^2}{2\gamma L}T - 1\right)\log\gamma\right)}{(\gamma - 1)\log\gamma},\]
\end{small}%
and where $W$ is the Lambert W-function \cite{corless_lambert}.
\end{theorem}
This leads to an amplification result, analogous to Thm.~\ref{thm:amplification_agg}:
\begin{theorem}
\label{thm:amplification_switch}
Let $\kappa$ be as in Thm.~\ref{thm:regret_switch}.  Any state whose advantage in the 
event-level optimal policy is at least $(1 + \frac{1}{1-\gamma})2\kappa L$
has an advantage of at least $2T$ in the switching-cost regularized optimal policy.
\end{theorem}

\subsection{Empirical Illustration}
\label{sec:empirical}

\begin{figure*}[ht]
    \centering
    \begin{subfigure}[b]{0.24\textwidth}
        \includegraphics[width=\textwidth]{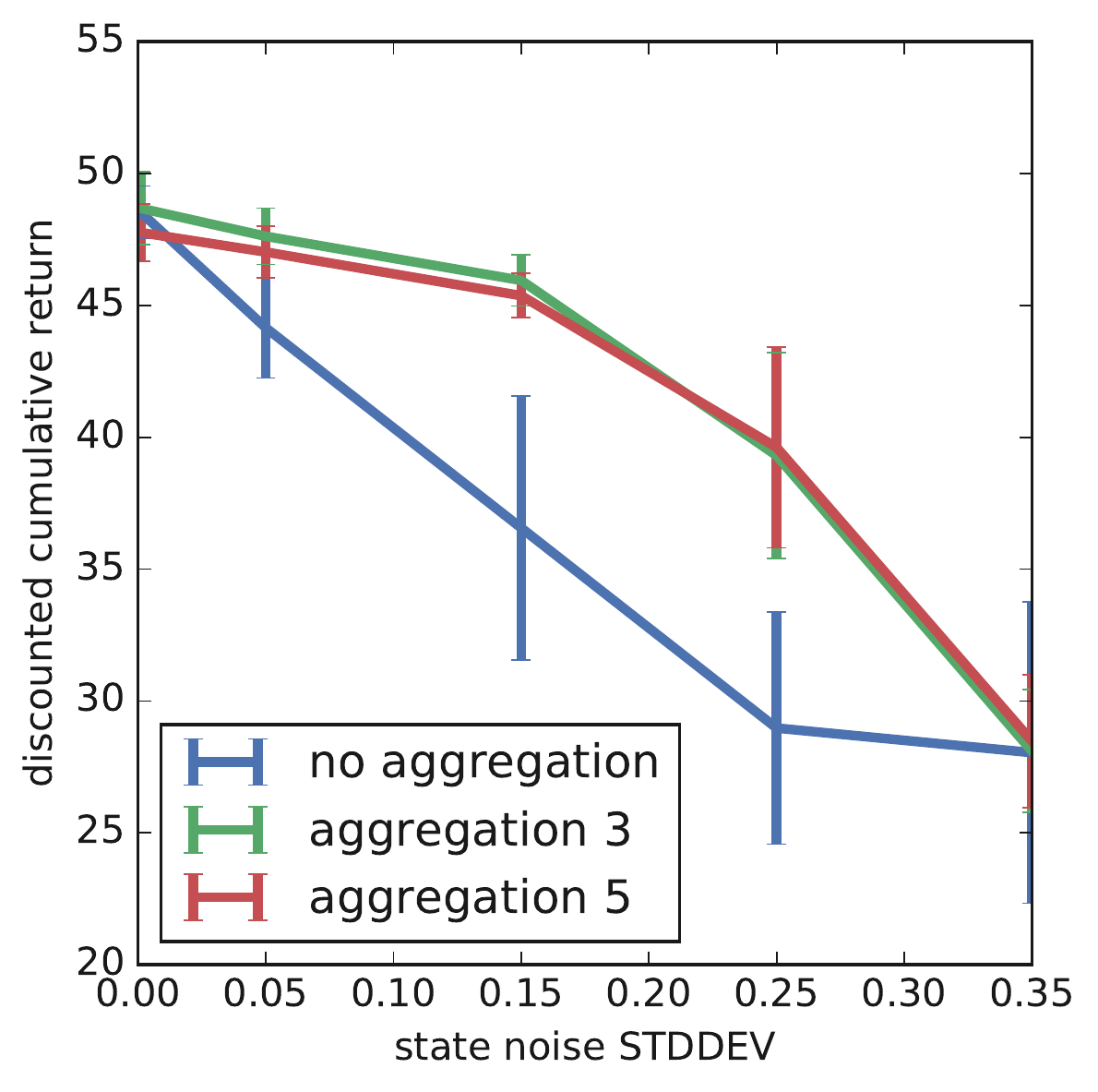}
        \label{fig:gull1}
    \end{subfigure}
    ~ 
    \begin{subfigure}[b]{0.24\textwidth}
        \includegraphics[width=\textwidth]{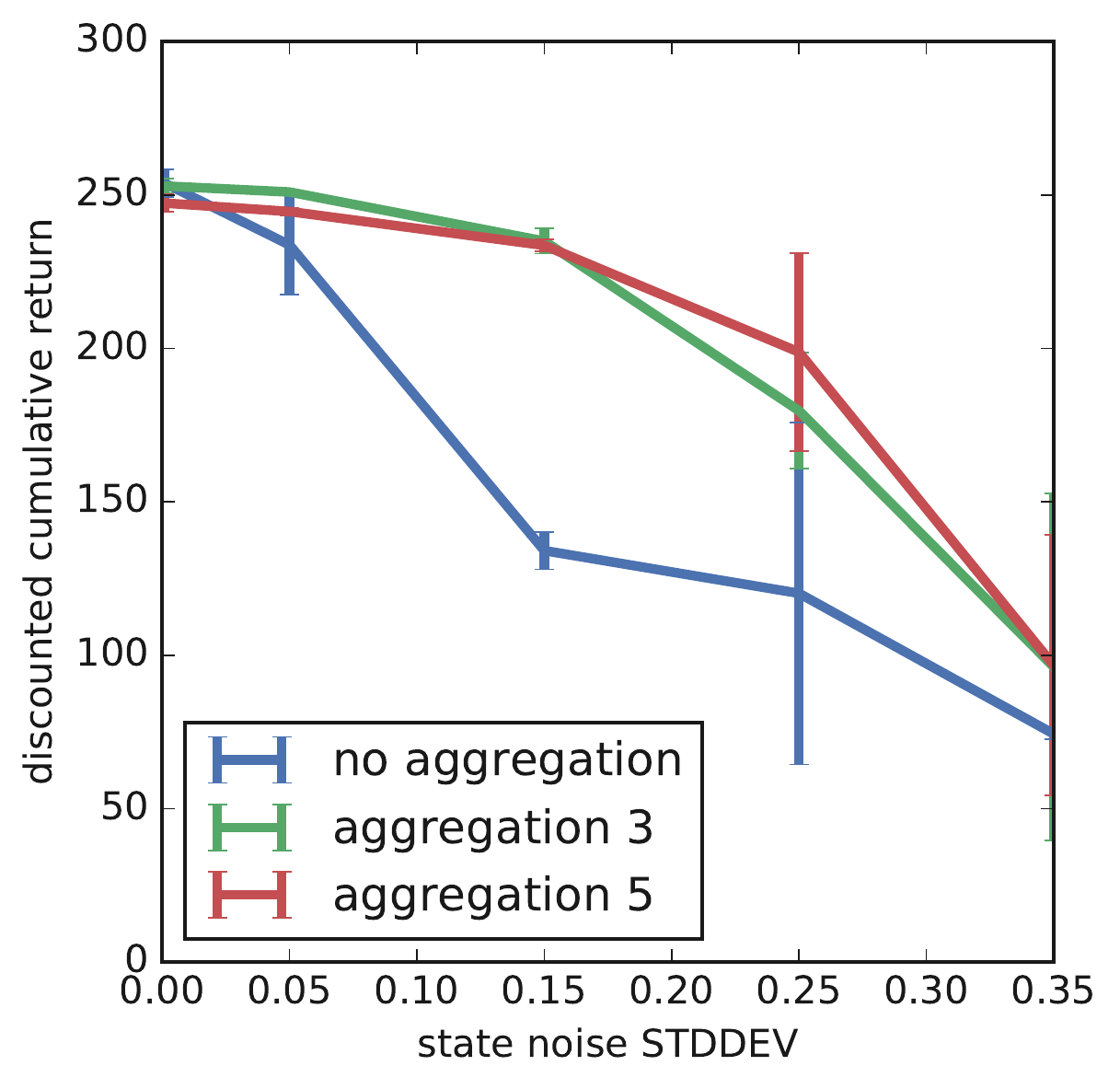}
        \label{fig:tiger1}
    \end{subfigure}
        \begin{subfigure}[b]{0.24\textwidth}
        \includegraphics[width=\textwidth]{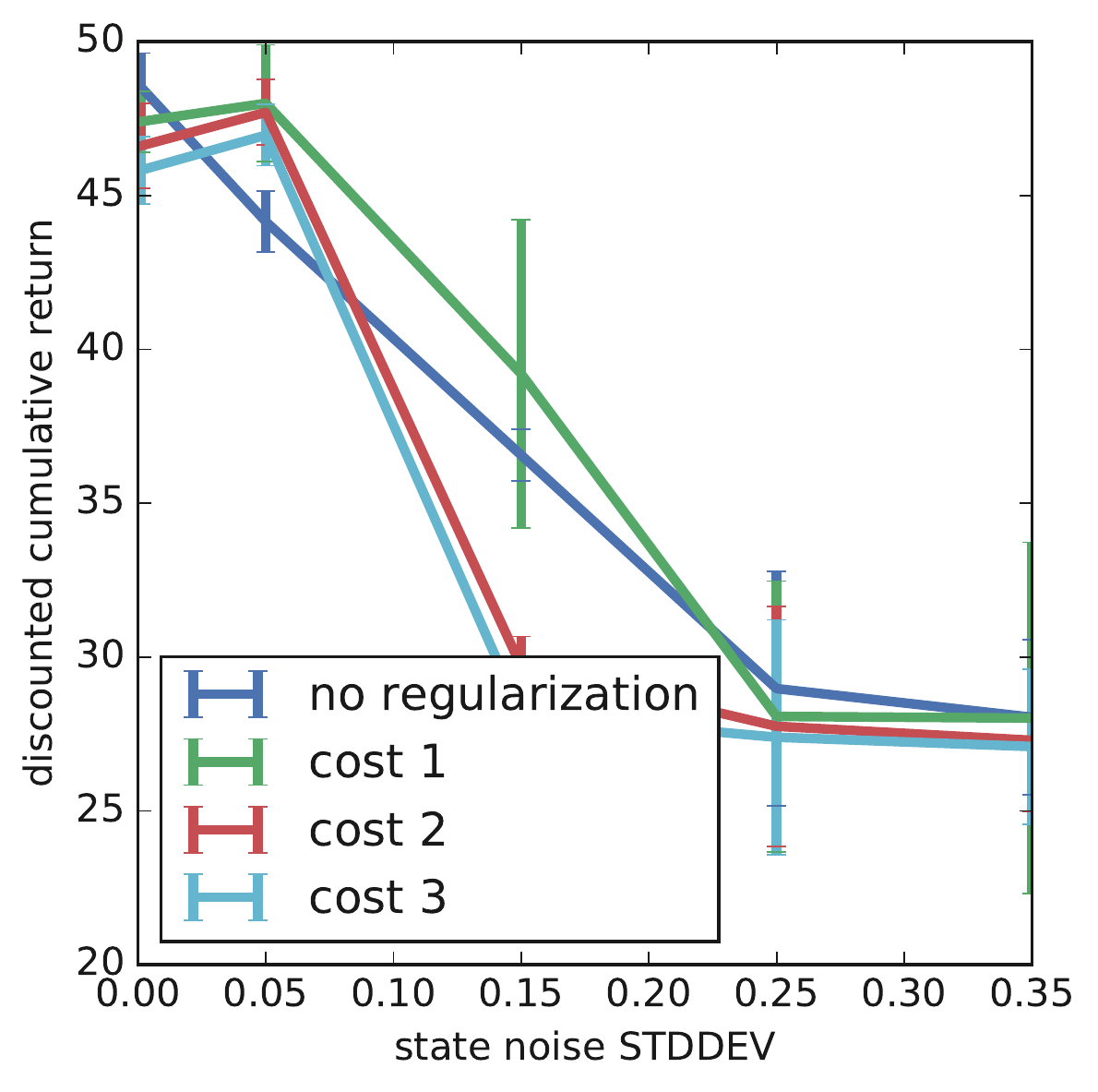}
        \label{fig:gull2}
    \end{subfigure}
    ~ 
    \begin{subfigure}[b]{0.24\textwidth}
        \includegraphics[width=\textwidth]{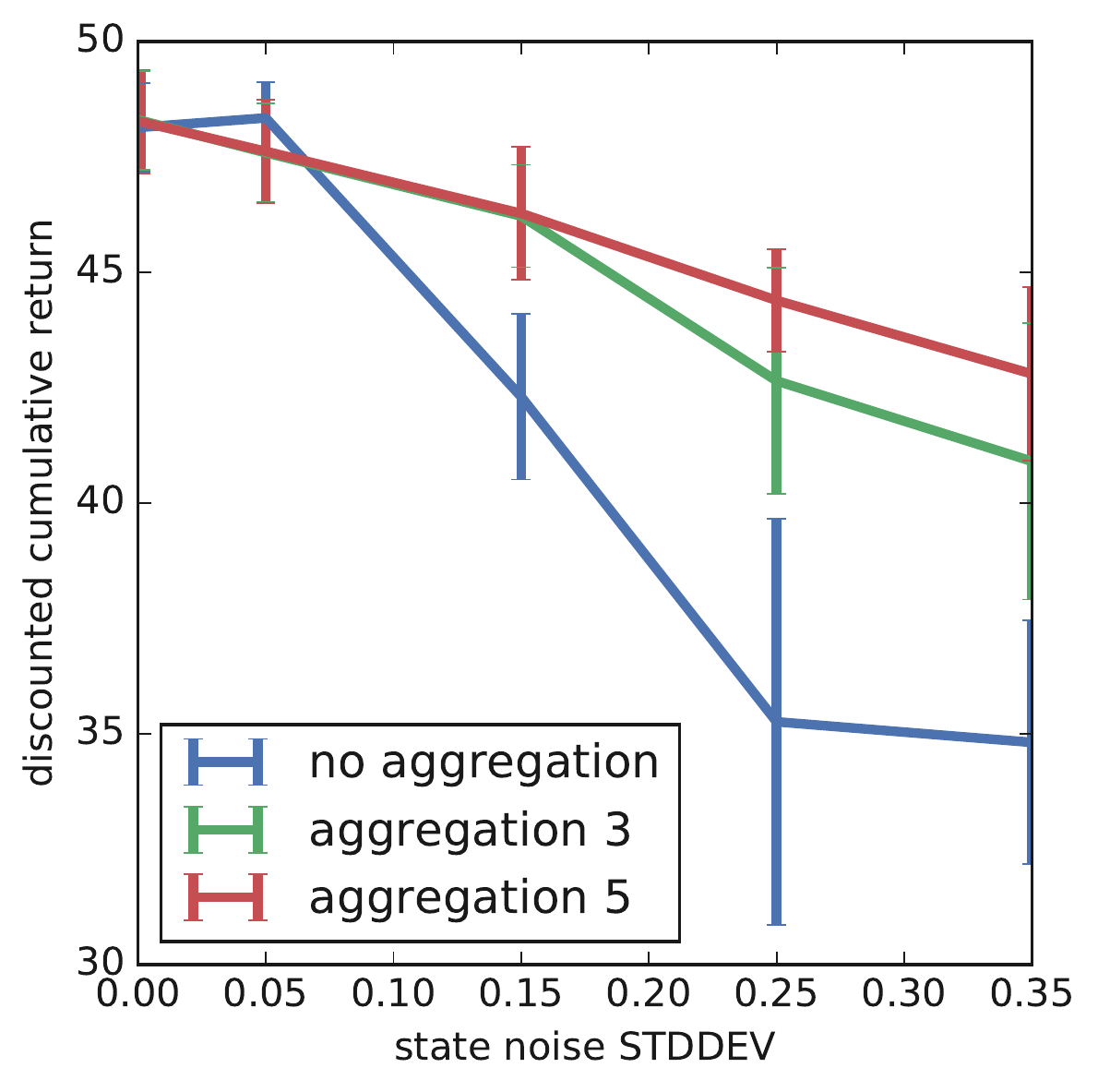}
        \label{fig:tiger2}
    \end{subfigure}
      \vspace*{-3mm}
      \caption{Experimental results.}\label{fig:experiment1}
      \vspace*{-3mm}
\end{figure*}

We experiment with synthetic models to demonstrate the theoretical results above. In a
first experiment, we apply both action aggregation and switching cost regularization to
the simple $\choc$-$\kale$ POMDP,
with parameters $\beta= 0.9, \tau = 0.25, \mu_{\mathrm{Choc}} = 8, \mu_{\mathrm{Kale}} = 2$, and $\sigma_{\mathrm{Choc}} = \sigma_{\mathrm{Kale}} = 0.5$. 
To illustrate the effects of faulty state estimation, we corrupt the satisfaction level $s$ with
noise drawn from a Gaussian (mean $0$, stdev.\ $\sigma_N$), truncated on $[-1, 1]$. As we increase $\sigma_N$,
state estimation becomes worse. To mitigate this effect, we apply aggregations of $3,5$ actions at discounts of $\gamma=0.95$ and $0.99$ (Fig.~\ref{fig:experiment1}a,b) and switching costs of $1, 2, 3$ (Fig.~\ref{fig:experiment1}c). For each parameter setting, we train $10$ tabular policies with $Q$-learning,
discretizing state space into $50$ buckets. For each training run, we 
roll-out $30000$ event-level transitions, exploring using actions taken uniformly at random---aggregated actions in the aggregation setting---then evaluate the discounted return of each policy using $100$ Monte Carlo rollouts of length $1000$.
Figs.\ref{fig:experiment1}a, b and c show the average performance across the $10$ training runs (with the 95\% confidence interval) as a function of the $\sigma_N$. 
We see that action aggregation has a profound effect on solution quality, improving the performance of the policy 
up to a factor of almost $2$ (for $\gamma=0.99$). Switching cost regularization has a more subtle effect, providing more modest improvements in performance. We observe that over-regularized policies actually perform worse than the unregularized policy. We conjecture that this stark difference in performance is due to action aggregation 
having a double effect on the value estimates---apart from amplification, it also provides a more favorable behavioral policy.

A second experiment takes a more ``options-oriented'' perspective on the problem. Here, recommendable items have a continuous ``kaleness'' score between 0 and 1, with item $i$'s score denoted $v(i)$. 
At each time step, a set of $7$ items is drawn from 
a $[0,1]$-truncated Gaussian with mean equal to the kaleness score of the previously
consumed item. The RL agent sets a \emph{target kaleness score} $\theta \in \{0, 0.25, 0.5, 0.75, 1\}$ (its action
space). This translates to a specific ``presentation'' of the $7$ items to the user such that
the user is nudged to consume an item whose score is closer to the target. Specifically,
the user chooses an item $i$ using a softmax distribution: $P(i) \propto \exp(-|v(i) - \theta|/\lambda)$, with
temperature $\lambda=0.2$.
The results are shown in Fig.~\ref{fig:experiment1}d and exhibit a comparable level of improvement as in the
binary-action case.


\section{Related Work}
\label{sec:related}

The study of time series at different scales of granularity has a long-standing history in econometrics, where the main object of interest seems to be the behavior of various autoregressive models under aggregation
\cite{silvestrini2008temporal}, however, the behavior of aggregated systems under control does not seem to have been investigated in that field.  

In RL, time granularity arises in several contexts. Classical semi-MDP/options theory employs temporal aggregation to organize the policy space into a hierarchy, where a pre-specified sub-policy, or \emph{option} is executed for
some period of time (termination is generally part of the option specification)
\cite{Sutton-etal:AIJ99}.
That options might help with partial observability 
(``state noise'') has been suggested---e.g., 
Daniel et al.~\scite{daniel_hireps}, who also informally suggest that reduced control frequency can improve SNR; however the task of characterizing this phenomenon formally has not
been addressed to the best of our knowledge. The \emph{learning to repeat} framework (see \cite{sharma2017} and references therein) provide a modeling perspective that allows an agent to choose an action-repetition granularity as part of the action space itself, but does not study these models theoretically.
SNR has played a role in RL, but in different ways than studied here, e.g., as applied to policy gradient (rather
than as a property of the domain) \cite{roberts:nips09}.

The effect of the advantage magnitude (also called action gap) on the quality and convergence of reinforcement learning algorithms was first studied by Farahmand~\scite{fahramand2011}. Bellemare et al.~\scite{bellemare2015} observed that the action gap can be manipulated to improve the quality of learned polices by introducing local policy consistency constraints to the Bellman operator. Their considerations are, however, not bound to specific environment properties.  

Finally, our framework is closely related with the study of regularization in RL and its benefits when dealing with POMDPs. Typically, an entropy-based penalty (or KL-divergence w.r.t.\ to a behavioral policy) is added
to the reward to induce a stochastic policy. This is usually justified in one of several ways:
inducing exploration \cite{nachum_etal:nips17}; accelerating optimization by making improvements monotone \cite{trpo};
and smoothing the Bellman equation and improving sample efficiency \cite{fox2016}. Of special relevance is the work of Thodoroff et al.~\scite{thodoroff2018}, who, akin to this work, exploit the sequential dependence of Q-values for better Q-value estimation. In all this work, however, regularization is simply a price to pay to achieve a side goal (e.g.,
better optimization/statistical efficiency). While stochastic policies often perform better than deterministic ones when state estimation is deficient \cite{singh:icml94}, and methods that exploit this have been proposed in restricting settings (e.g., corrupted rewards \cite{everitt2017}), the connection to regularization has not been made explicitly to the best of our knowledge.


\section{Concluding Remarks}
\label{sec:conclude}

We have developed a framework for studying the impact of belief-state approximation in latent-state
RL problems, especially suited to slowly evolving, highly noisy (low SNR) domains like recommender
systems. We introduced \emph{advantage amplification} and proposed and analyzed
two conceptually simple realizations of it. Empirical study on a stylized domain demonstrated the
tradeoffs and gains they might offer. 

There are a variety of interesting avenues suggested by this work: (i) the study of soft-policy regularization
for amplification (preliminary results are presented in the long version of this paper); (ii) developing 
techniques for constructing more general ``options'' (beyond aggregation) for amplification; (iii) developing
amplification methods for arbitrary sources of modeling error; (iv) conducting more extensive empirical
analysis on real-world domains.


\input{AA_ijcai.bbl}

\clearpage
\input{appendix}

\end{document}

%% file: appendix.tex
\appendix
\section{Technical Results and Additional Material}
\subsection{Analysis of Action Aggregation}
We start with an observation that frames the subsequent analysis.
\begin{lemma}[Counterfactual Q-values]
\label{lem:qvalues}
Let $\pi$ and $\rho$ be two deterministic policies and let the $Q$-function of $\pi$, $Q^\pi$ be known. Then, for any state $\mathfrak{b}\in\calB$:
\begin{align*}
    V^\pi(\mathfrak{b}) -& V^\rho(\mathfrak{b}) \\
    &= Q^\pi(\mathfrak{b}, \pi(\mathfrak{b})) - Q^\rho(\mathfrak{b}, \rho(\mathfrak{b})) \\
    &= \E_\rho\left[\sum_{i=0}\gamma^i \left(Q^\pi(\mathfrak{b}_i, \pi(\mathfrak{b}_i)) - Q^\pi(\mathfrak{b}_i, \rho(\mathfrak{b}_i))\right)\right],
\end{align*}
where $\E_\rho$ the expectation is over trajectories $(\mathfrak{b}_i)_{i\in \mathbb{N}}$ generated by $\rho$ starting at $\mathfrak{b}$ ($\mathfrak{b}_0 = \mathfrak{b}$).
\end{lemma}
\begin{proof}
Let $\rho\pi$ be a non-stationary policy that starting at $\mathfrak{b}$ executes $\rho(\mathfrak{b})$ exactly once, then follows $\pi$ forever. Then,
\begin{align*}Q^\pi(\mathfrak{b}, \pi(\mathfrak{b})) &- Q^\rho(\mathfrak{b}, \rho(\mathfrak{b})) \\
= &\quad Q^\pi(\mathfrak{b}, \pi(\mathfrak{b})) - Q^{\rho\pi}(\mathfrak{b}, \rho\pi(\mathfrak{b}))\\ 
&+ Q^{\rho\pi}(\mathfrak{b}, \rho\pi(\mathfrak{b})) - Q^\rho(\mathfrak{b}, \rho(\mathfrak{b}))\\
= &\quad Q^\pi(\mathfrak{b}, \pi(\mathfrak{b})) - Q^{\pi}(\mathfrak{b}, \rho(\mathfrak{b}))\\ 
&+ Q^{\pi}(\mathfrak{b}, \rho(\mathfrak{b})) - Q^\rho(\mathfrak{b}, \rho(\mathfrak{b}))\\
=&\quad Q^\pi(\mathfrak{b}, \pi(\mathfrak{b})) - Q^{\pi}(\mathfrak{b}, \rho(\mathfrak{b})) \\
&+ \gamma \E_{\mathfrak{b}^\prime \sim \rho}[Q^\pi(\mathfrak{b}^\prime, \pi(\mathfrak{b}^\prime)) - Q^\rho(\mathfrak{b}^\prime, \rho(\mathfrak{b}^\prime))].
\end{align*}
Linearity of expectation and the boundedness of both $Q$ functions ensures that recursive application to $Q^\pi(\mathfrak{b}^\prime, \pi(\mathfrak{b}^\prime)) - Q^\rho(\mathfrak{b}^\prime, \rho(\mathfrak{b}^\prime))$ converges to the desired quantity.
\end{proof}
This allows us to measure the behavior of $k$-aggregate policies (in terms of advantages) relative to atomic ones.

\begin{lemma}
\label{lem:qvalues_repetition}
Let $\mathfrak{b}$ be a state for which the optimal action $a$ does not change within $k$ steps. In other words, for all $\mathfrak{b}^\prime \in {\cal B}^{\times k}(\mathfrak{b})$, $\pi^*(\mathfrak{b}^\prime) = a$, where ${\cal B}^{\times k}(\mathfrak{b})$ denotes all states reachable from $\mathfrak{b}$ in $k$ steps under some action sequence. Then for any other action $a^\prime$,
\begin{align*}
Q^\ast(\mathfrak{b}, a^{\times k}) - & Q^\ast(\mathfrak{b}, a^{\prime\times k})\\
&= \E \left[\sum_{i=0}^k\gamma^i \left(Q^\ast(\mathfrak{b}^{a^\prime}_i, a) - Q^\ast(\mathfrak{b}^{a^\prime}_i, a^\prime)\right)\right]
\end{align*}
where the expectation is over the trajectory $\mathfrak{b}^{a^\prime}_1,\ldots, \mathfrak{b}^{a^\prime}_k$ of states that follow after $\mathfrak{b}$ when taking action ${a^\prime}$ and $\mathfrak{b}^{a^\prime}_0 = \mathfrak{b}$ for notational convenience.
\end{lemma}
\begin{proof}
Let $\pi$ be a policy that executes $a$ for $k$ steps and then reverts to the optimal policy and $\rho$ be a policy that executes $b$ for $k$ steps and then reverts to the optimal policy. We apply Lemma~\ref{lem:qvalues} noting that $\pi$ coincides with the optimal policy, hence $Q^\pi = Q^\ast$.
\end{proof}
Lemma \ref{lem:qvalues_repetition} establishes that advantage of $k$-aggregate actions is the compound discounted advantage of the atomic ones. This, combined with the smoothness of the optimal $Q$-function allows us to analyze the effects of action aggregation. In particular, suppose that we have found a state $\mathfrak{b}$ such that $A(\mathfrak{b}) = Q^\ast(\mathfrak{b}, a) - \max_{a^\prime\neq a}Q^\ast(\mathfrak{b},a^\prime)\geq 2kL$. The smoothness of the $Q$-function allows us to infer that for any action taken, the advantage at the next state will be at least $2kL - 2L$, $2kL - 4L$ after $2$ steps and so on (this guarantees the advantage gap can't close in less than $k$ steps). Hence if we were to replace atomic actions with $k$-repeated actions at states with advantage of more than $2kL$, the following can be observed:  
\begin{lemma}
\label{lem:losslessamp}
Consider an MDP with an $L$-smooth optimal $Q$-function $Q^\ast$, and the reparametrization by holding actions fixed for $k$ steps whenever the advantage gap is greater than $2kL$. Then, at each state, either:
\begin{align*}
A^{\times k}(\mathfrak{b}) :=&~ Q^\ast(\mathfrak{b}, a^{\times k}) - Q^\ast(\mathfrak{b}, a^{\prime\times k})\\
\geq&~ A(\mathfrak{b})\frac{1-\gamma^{k}}{1-\gamma}\\ &- 2\gamma L\frac{1 - (1 + k - \gamma k)\gamma^{k}}{(1-\gamma)^2},
\end{align*}
or $A(\mathfrak{b}) \leq 2kL$.
\end{lemma}
\begin{proof}
As established, $A^{\times k}(\mathfrak{b})  = E[\sum_{i=0}^k\gamma^i A(\mathfrak{b}_i^{a^\prime})]$. Smoothness of the $Q$-function guarantees that $A(\mathfrak{b}_{i+1}^{a^\prime})\geq A(\mathfrak{b}_i^{a^\prime}) - 2L$, resp. $A(\mathfrak{b}_k^{a^\prime})\geq A(\mathfrak{b}) - 2kL$ as discussed above. Hence
\begin{align}
A^{\times k}(\mathfrak{b}) &= \E \left[ \sum_{i=0}^k\gamma^i A(\mathfrak{b}_i^b) \right]\\
&\geq \E \left[ \sum_{i=0}^k\gamma^i(A(\mathfrak{b}) - 2iL) \right]\\
&= A(\mathfrak{b})\frac{1-\gamma^{k}}{1-\gamma} - 2L\frac{\gamma - (1 + k - \gamma k)\gamma^{k+1}}{(1-\gamma)^2}.
\end{align}
Here, the first expression comes from the finite sum geometric series formula, and the second term reflects the fact that \[\sum_{i=0}^k i\gamma^i = \frac{\gamma - (1+k - \gamma k)\gamma^{k+1}}{(1-\gamma)^2}.\]
\end{proof}

\ignore{
\begin{figure*}[t]
\centering
  \includegraphics[width=0.5\linewidth]{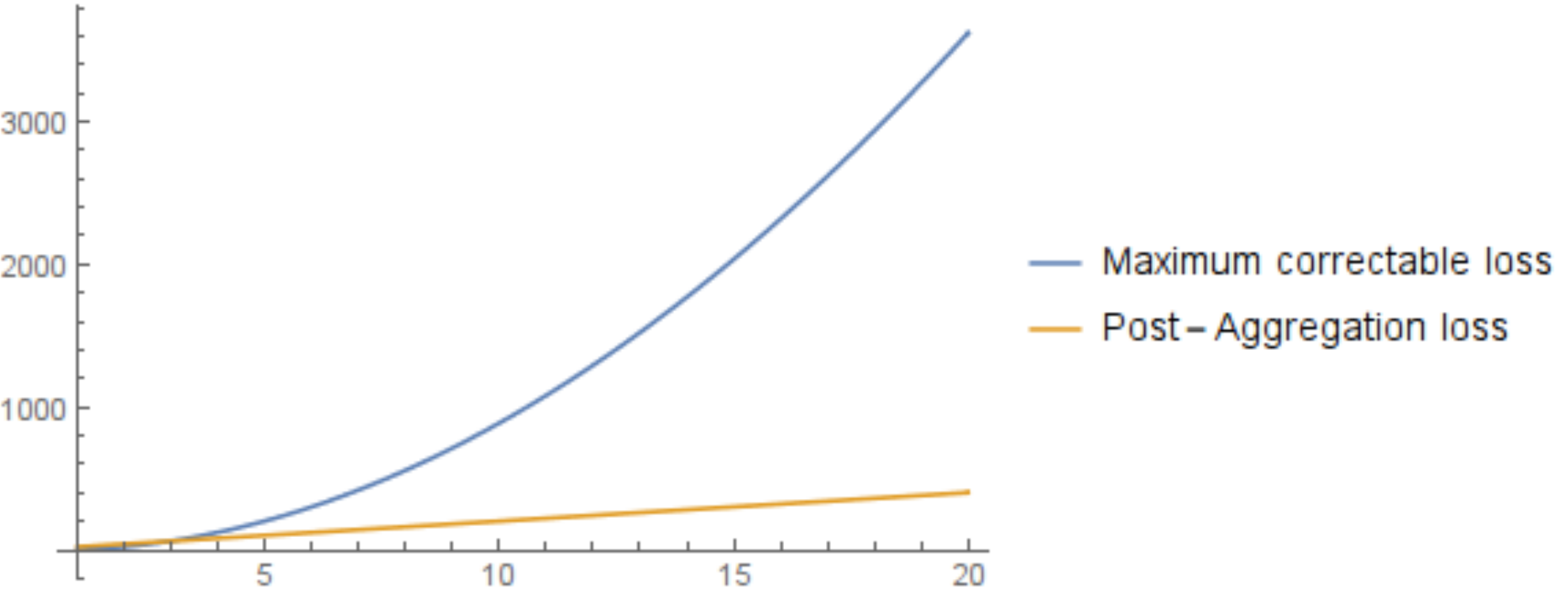}
   \vspace*{-2mm}
   \caption{\todo{better plot}}
  \vspace*{-3mm}
  \label{fig:ideal_noise_rejection}
\end{figure*}
}
We have thus far considered the situation where aggregation is state-dependent, based on the knowledge of the advantage amplitude. A more realistic implementation of aggregation would be to fix some global static $k$ apriori. While the reasoning is similar to the above, there is an added complexity that the aggregation itself comes at a price -- the values of certain states will be reduced (relative to the best atomic policy) due to the inability to rapidly switch actions. This additional cost must then be factored into the bounds.

We first bound the cost of the best slow policy under smooth $Q$ assumptions. 
\begin{customthm}{\ref{thm:approxbound}}
Let $k$ be a fixed horizon and $Q^\ast$, the event-level optimal Q-function be $L$-smooth. Then for all $\mathfrak{b}$, $|V^\ast(\mathfrak{b}) - V^{\times k\ast}(\mathfrak{b})| \leq \frac{2kL}{1-\gamma}$, where $ V^{\times k\ast}(\mathfrak{b})$ is the value of state $\mathfrak{b}$ under an optimal $k$-aggregate policy\footnote{Note that the reparametrized problem in which a decision can be made every $k$ atomic steps is also an MDP, so the notion of an optimal value function (one that provides the largest possible value at every state) resp. optimal deterministic policy is well-defined. }.
\end{customthm}
\begin{proof}
We lower bound the return of the optimal $k$-aggregate policy by exhibiting a (not-necessarily optimal) $k$-aggregate policy with sufficient returns. In particular, let $\pi$ be a deterministic optimal atomic policy, and $\rho$ be the policy that at state $\mathfrak{b}_t$ executes $\pi(\mathfrak{b}_t)$ for the next $k$ steps, then executes $\pi(\mathfrak{b}_{t+k})$, etc. Following Lemma~\ref{lem:qvalues}, we have that 
\begin{align*}
V^\pi(\mathfrak{b}) - & V^\rho(\mathfrak{b}) \\
&= \E_\rho\left[\sum_{i=0}\gamma^i \left(Q^\pi(\mathfrak{b}_i, \pi(\mathfrak{b}_i)) - Q^\pi(\mathfrak{b}_i, \rho(\mathfrak{b}_i))\right)\right].    
\end{align*}

Let $(\mathfrak{b}_{ik})_{i\in \mathbb{N}}$ be the states at which $\rho$ may switch actions. At any $\mathfrak{b}_{ik}$, 
if $A(\mathfrak{b}_{ik})\geq 2kL$, due to smoothness of $Q^*$, we know that $a$ remains optimal for the next $k$ steps thus $\rho(\mathfrak{b}_{ik + j}) = \pi(\mathfrak{b}_{ik + j})$ and $Q^*(\mathfrak{b}_{ik + j}, \pi(\mathfrak{b}_{ik+j})) - Q^*(\mathfrak{b}_{ik+j}, \rho(\mathfrak{b}_{ik+j})) = 0$ until the next decision point for $\rho$ (i.e. for $j\leq k - 1$). Conversely, if $A(\mathfrak{b}_{ik}) < 2kL$, then in the worst case, assuming that the Q-value of the optimal action $Q^*(\mathfrak{b}_{ik+j},a)$ decreases by $L$ with every step $j$ and the Q-value of a suboptimal action $Q^*(\mathfrak{b}_{ik+j}, a^\prime)$ increases respectively by $L$, then for any of the next $k$ steps, $Q^*(\mathfrak{b}_{ik+j},a) - Q^*(\mathfrak{b}_{ik+j}, a^\prime) \geq -2kL$. Hence, for all $i$, $|Q^*(\mathfrak{b}_i, \pi(a)) - Q^*(\mathfrak{b}_i, \rho(a)))|$ is either $0$ or less than $2kL$, leading to 
\[V^\pi(\mathfrak{b}) - V^\rho(\mathfrak{b}) \leq \sum_{i=0}\gamma^i 2kL = \frac{2kL}{1-\gamma}.\]
\end{proof}

We now factor in the aggregation loss from Thm.~\ref{thm:approxbound} into Lemma~\ref{lem:losslessamp}, which together characterize the amplification properties of action aggregation.
\begin{customthm}{\ref{thm:amplification_agg}}
In an $L$ smooth MDP, let $k$ be a fixed repetition horizon. For any state where the advantage gap $A(\mathfrak{b})$ greater than $2kL$, the fixed-horizon advantage is lower-bounded as follows:
\begin{align*}A(\mathfrak{b})^{\times k}\\
\geq&~ A(\mathfrak{b})\frac{1-\gamma^{k}}{1-\gamma} - 2L\frac{\gamma - (1+k - \gamma k)\gamma^{k+1}}{(1-\gamma)^2}\\
&- \frac{2kL}{1-\gamma}.\end{align*}
\end{customthm}
\begin{proof}
Let $\bar{Q}^{\times k\ast}$ be the losslessly amplified $Q$-function as in Lemma~\ref{lem:losslessamp}. Since $A(\mathfrak{b}) \geq 2kL$, action $a$ is optimal under an atomic optimal policy for the next $k$ periods and due to Thm.~\ref{thm:approxbound},
\[Q^{\times k\ast}(\mathfrak{b}, a^{\times k}) \geq  Q^\ast(\mathfrak{b}, a) - \frac{2kL}{1-\gamma} = \bar{Q}^{\times k\ast}(\mathfrak{b}, a^{\times k})- \frac{2kL}{1-\gamma}.\]
Moreover, for any $a^\prime \neq a$
\[Q^{\times k\ast}(\mathfrak{b}, a^{\prime\times k}) \leq \bar{Q}^{\times k\ast}(\mathfrak{b}, a^{\prime\times k}),\]
as the expected return following the aggregation period of $k$ can only be lower than the one of the losslessly amplified problem. Thus:
\begin{align*}
Q^{\times k\ast}(\mathfrak{b}, a^{\times k}) & - Q^{\times k\ast}(\mathfrak{b}, a^{\prime\times k})\\
\geq&~ \E\left[ \sum_i^k \gamma^i \left(Q^\ast(\mathfrak{b}_i, a) - Q^{\ast}(\mathfrak{b}_i, a^\prime)\right) \right]  - \frac{2kL}{1-\gamma}.
\end{align*}
The bound from the expectation term comes from Lemma~\ref{lem:losslessamp}.
\end{proof}

\subsection{Analysis of Switching Cost}

Let $Q^*$ be an L-smooth optimal Q-function of a POMDP with an optimal deterministic policy $\pi^*$ and $T$ be a switching cost. Consider the following policy $\rho$: at $\mathfrak{b}_0$, $\rho$ adopts the optimal atomic action $a_0 = \rho_0(\mathfrak{b}_0)=\pi^*(\mathfrak{b}_0)$. Also at $t=0$, $\rho$ calculates the time until its regret for repeating $a_0$ relative to following the optimal policy exceeds, $T$, i.e.

\begin{align*}
\omega_0 & = \\
&\argmin_t \left\{t : \sum_{i=0}^t \gamma^i\mathbb{E}\left[Q^*(\mathfrak{b}_i, \pi^*(\mathfrak{b}_i))\right.\right.\\
 &\quad\quad\quad\quad\quad\quad\quad\quad\quad\quad\quad\quad\left.-Q^*(\mathfrak{b}_i, a_0)| \mathfrak{b}_0\right] \geq T \Bigg\} ,
\end{align*}%
(where the expectation is, as before, over trajectories of belief sates generated by executing $a_0$, conditioned on the realization of $\mathfrak{b}_0$)
and then repeats $a_0$ until $t = \omega_0$. At $t=\omega_0$, $\rho$ queries the optimal policy for $a_{\omega_0} = \pi^*(\mathfrak{b}_{\omega_0})$, executes it until time $\omega_1$ and so on. In other words, $\rho$ repeats an action adopted from the optimal policy until its regret for not having followed the optimal policy exceeds $T$ and then switches. 

The return of $\rho$ is equivalent to the return of a policy that would pay $T$ upfront to switch to and follow the optimal policy for $\omega$ steps. We first bound the loss of this policy and then argue that it also upper-bounds the loss of the optimal switching cost policy. 

\begin{lemma}
\label{lem:regretbound}
The total regret of $\rho$ relative to $\pi^*$ is no more than \[\frac{2L}{1-\gamma}\frac{\log\gamma + (\gamma - 1)W\left(\frac{\gamma^{1/(1-\gamma)}}{\gamma-1}\left(\frac{(1-\gamma)^2}{2\gamma L}T - 1\right)\log\gamma\right)}{(\gamma - 1)\log\gamma}.\]
\end{lemma}

\begin{proof}
Starting from Lemma~\ref{lem:qvalues}, the total loss of $\rho$ relative to $\pi^*$ is 
\[\sum_{i=0} \gamma^i\mathbb{E}_{\mathfrak{b}_i \sim \rho}\left[Q^*(\mathfrak{b}_i, \pi^*(\mathfrak{b}_i))-Q^*(\mathfrak{b}_i, \rho_i)\right].\] We now argue that $\mathbb{E}_{\mathfrak{b}_i \sim \rho}\left[Q^*(\mathfrak{b}_i, \pi^*(\mathfrak{b}_i))-Q^*(\mathfrak{b}_i, \rho_i)\right]$ is bounded at any step $i$. Observe that the action taken by the policy at time $i$ depends only on the belief state realization  $\mathfrak{b}_{\omega_{-1}}$, $\omega_{-1}$ being the most recent time point at which the policy was allowed to switch actions. Thus, by the law of total probability,
\begin{align*}
    \mathbb{E}&_{\mathfrak{b}_i \sim \rho}\left[Q^*(\mathfrak{b}_i, \pi^*(\mathfrak{b}_i))-Q^*(\mathfrak{b}_i, \rho_i)\right]\\
    &= \mathbb{E}_{\mathfrak{b}_{\omega_{-1}}\sim\rho}\mathbb{E}_{\mathfrak{b}_i \sim \rho}\left[Q^*(\mathfrak{b}_i, \pi^*(\mathfrak{b}_i))-Q^*(\mathfrak{b}_i, \rho_i)|\mathfrak{b}_{\omega_{-1}}\right].
\end{align*}
In the above, we have just rewritten the loss relative to the last switching point $\omega_{-1}$. We now show that the conditional expectation $\mathbb{E}_{\mathfrak{b}_i \sim \rho}\left[Q^*(\mathfrak{b}_i, \pi^*(\mathfrak{b}_i))-Q^*(\mathfrak{b}_i, \rho_i)|\mathfrak{b}_{\omega_{-1}}\right]$ is bounded for any realization $\mathfrak{b}_{\omega_{-1}}$. 
Recall that at time $\omega_{-1}$, $\rho$ picks the action $a_{\omega{-1}} = \pi^*(\omega_{-1})$ and executes it until time $\omega$, the time when its expected regret exceeds $T$, and then switches. Thus, the highest achievable per-event expected regret conditioned on $\mathfrak{b}_{\omega_{-1}}$ occurs when at time $\omega_{-1}$, there exists an alternative action $a^\prime \neq a_{\omega{-1}}$ having the same Q-value, i.e. $Q^*(\mathfrak{b}_{\omega_{-1}}, a_{\omega_{-1}}) = Q^*(\mathfrak{b}_{\omega_{-1}}, a^\prime)$,
 and then, in subsequent time steps, the expected $Q$-value of $a^\prime$, $\mathbb{E}_{\mathfrak{b}_i \sim \rho}\left[Q^*(\mathfrak{b}_i, a^\prime)|\mathfrak{b}_{\omega_{-1}}\right]$, starts increasing at the maximum  rate of $L$ while the Q-value of $a_{\omega_{-1}}$ starts decreasing at the same rate. The maximum rate of increase of the expectation is justified since if the Q-value of $a^\prime$ is $L$-smooth over individual sample paths, i.e. $|Q^*(\mathfrak{b}_i, a) - Q^*(\mathfrak{b}_{i+1}, a)| \leq L$, the sequence of expectations is also $L$-smooth, $|\mathbb{E}_{\mathfrak{b}_i \sim \rho}
\left[Q^*(\mathfrak{b}_i, a)|\mathfrak{b}_{\omega_{-1}}\right] - \mathbb{E}_{\mathfrak{b}_{i+1}|\mathfrak{b}_{\omega_{-1}} \sim \rho}\left[Q^*(\mathfrak{b}_{i+1}, a)|\mathfrak{b}_{\omega_{-1}}\right]| \leq L$). By the same reasoning as in Lemma~\ref{lem:losslessamp} (but this time for the sequence of expectations, rather than the sequence of realizations), the regret for not following $\pi^*$ after $k$ steps is bounded by $2\gamma L\frac{1 + k\gamma^k - (1+k)\gamma^k}{(1-\gamma)^2}$. Let us calculate the minimum $k$ (under assumptions of monotone regret increase at the maximum rate)
for which this regret exceeds $T$ as $\lceil \kappa \rceil$, where $\kappa$ is the solution of  $2\gamma L\frac{1 + k\gamma^k - (1+k)\gamma^k}{(1-\gamma)^2} = T$.
\begin{align*}
    2\gamma & L\frac{1 + k\gamma^{k + 1} - (1+k)\gamma^k}{(1-\gamma)^2} = T\\
    \rightarrow & k\gamma^{k + 1} - (1+k)\gamma^k = \frac{(1-\gamma)^2}{2\gamma L}T - 1\\
    \rightarrow & (\gamma - 1) k\gamma^{k} - \gamma^k = \frac{(1-\gamma)^2}{2\gamma L}T - 1\\
    \rightarrow & k = \frac{\log\gamma + (\gamma - 1)W\left(\frac{\gamma^{1/(1-\gamma)}}{\gamma-1}\left(\frac{(1-\gamma)^2}{2\gamma L}T - 1\right)\log\gamma\right)}{(\gamma - 1)\log\gamma},
\end{align*}
where $W$ is the Lambert W-function.
Knowing $\kappa$, we can deduce that the maximum per-event expected regret, $\mathbb{E}_{s_i \sim \rho}\left[Q^*(s_i, \pi^*(s_i))-Q^*(s_i, \rho_i)|\mathfrak{b}_{\omega_{-1}}\right]$ is no more than
$$\epsilon = 2L \left\lceil \frac{\log\gamma + (\gamma - 1)W\left(\frac{\gamma^{1/(1-\gamma)}}{\gamma-1}\left(\frac{(1-\gamma)^2}{2\gamma L}T - 1\right)\log\gamma\right)}{(\gamma - 1)\log\gamma} \right\rceil,$$ under worst-case assumptions. This
leads to an overall regret bound of $\epsilon/(1-\gamma)$.
\end{proof}

So far, we have produced an agent that pays a cost $T$ to switch to the optimal policy for some time-specific horizon $\omega_t$ and bounded its regret. This is not quite what we need, since $\rho$ gets more than one action switch for free within the horizon over which it is allowed to follow the optimal policy. We now argue that for a slow environment, this bound also holds for an agent which pays on every switch. To do so, we need to produce a reasonable policy that pays for every action switch.

Let $A(\mathfrak{b},\pi^*, a^{\times k})$ be the advantage of following the optimal policy for $k$ turns over repeating $a$ for the same time and then switching to the optimal policy (resp. the regret for repeating $a$ over following the optimal policy). That is, following Lemma~\ref{lem:qvalues},   
\begin{align*}
    A&(\mathfrak{b},\pi^*, a^{\times k}) =\\
    &\mathbb{E}_{\mathfrak{b}_0,\ldots,\mathfrak{b}_k}\left[\sum_{i=0}^k\gamma^{i}(Q^\ast(\mathfrak{b}_i, \pi^\ast(\mathfrak{b}_i))-Q^\ast(\mathfrak{b}_i, a))\right],
\end{align*} where the expectation is over trajectories generated by taking action $a$ and $\mathfrak{b}_0 = \mathfrak{b}$. 

Furthermore, let  $\omega(\mathfrak{b}) = \argmin_k\min_{a^\prime\in A} k$ s.t. $A(\mathfrak{b}, \pi^\ast, a_0^{\times k}) - A(\mathfrak{b}, \pi^\ast, a^{\prime\times k}) \geq T$ if the former is feasible, else $\infty$. In words, $\omega(\mathfrak{b})$ is the shortest time until the regret of having repeated the current action $a_0$ exceeds the regret of having repeated some other action $a^\prime$ by $T$. Note that this does not need to be finite, since in some state, repeating any action for any amount of time might yield similar returns.  

Based on this, we construct a policy $\sigma$ that behaves as if it would pay $T$ on every switch. 
At $\mathfrak{b}_0$, $\sigma$ executes the atomic optimal action $\pi^*(\mathfrak{b}_0)$. At the next state, $\mathfrak{b}_1$, if $\omega(\mathfrak{b}_1)$ is infinite, $\sigma$ repeats $a_0$ once. Otherwise, $\sigma$ repeats $a_0$ $\omega(\mathfrak{b}_1)$ times and then switches to $a^\prime$, the action whose regret exceeds $T$ after $\omega(\mathfrak{b}_1)$ steps. This policy can do no better than a policy that pays $T$ upfront to switch to $a^\prime$ and maintain it for  $\omega(\mathfrak{b}_1)$ steps. By extension, $\sigma$ can generate no more return than the best switching cost policy. Now we can bound its loss.

\begin{customthm}{\ref{thm:regret_switch}}
The regret of the optimal switching cost policy for a $2$-action MDP is less than $\frac{2\kappa L}{1-\gamma}$, where $\kappa$ is the same as in the previous theorem.
\end{customthm}
\begin{proof}
Again we must scrutinize the sequence $\left(\mathbb{E}_{\mathfrak{b}_i\sim \sigma}\left[Q^*(\mathfrak{b}_i, \pi^*(\mathfrak{b}_i))-Q^*(\mathfrak{b}_i, \sigma_i)\right]\right)_{i\in\mathbb{N}}$ showing that it is bounded at any step $i$. We consider the contrapositive: suppose there existed $i$, such that
\[\mathbb{E}_{\mathfrak{b}_i\sim \sigma}\left[Q^*(\mathfrak{b}_i, \pi^*(\mathfrak{b}_i))-Q^*(\mathfrak{b}_i, \sigma_i)\right] > 2\kappa L\]
with $\kappa$ as in the previous theorem.
By the law of total probability,
\begin{align*}
    \mathbb{E}&_{\mathfrak{b}_i\sim \sigma}\left[Q^*(\mathfrak{b}_i, \pi^*(\mathfrak{b}_i))-Q^*(\mathfrak{b}_i, \sigma_i)\right]\\
    & = \mathbb{E}_{\mathfrak{b}_{i-\kappa}}\mathbb{E}_{\mathfrak{b}_i\sim \sigma} \left[Q^*(\mathfrak{b}_i, \pi^*(\mathfrak{b}_i))-Q^*(\mathfrak{b}_i, \sigma_i)|\mathfrak{b}_{i-\kappa}\right] > 2\kappa L.
\end{align*}
The expectation over $\mathfrak{b}_{i-\kappa}$ in the second line above implies that for at least one possible realization of $\mathfrak{b}_{i-\kappa}$, the expected Q difference after $\kappa$ periods exceeds $2\kappa L$, i.e. $\mathbb{E}_{\mathfrak{b}_i\sim \sigma} \left[Q^*(\mathfrak{b}_i, \pi^*(\mathfrak{b}_i))-Q^*(\mathfrak{b}_i, \sigma_i)| \mathfrak{b}_{i-\kappa}\right]\geq 2\kappa L$. Let $\hat{\mathfrak{b}}$ be such a realization of $\mathfrak{b}_{i-\kappa}$ and let $a^* = \pi^*(\mathfrak{b}_i)$ resp. $a = \sigma(\mathfrak{b}_i)$ . Due to the $L$-smoothness of the sequences of expected Q-values $\left(\mathbb{E}_{\mathfrak{b}_j \sim \rho}\left[Q^*(\mathfrak{b}_j, a^*)|\mathfrak{b}_{i-\kappa}=\hat{\mathfrak{b}}\right]\right)_{j\in [i-\kappa, \ldots, i]}$, $\left(\mathbb{E}_{\mathfrak{b}_j \sim \rho}\left[Q^*(\mathfrak{b}_j, a)|\mathfrak{b}_{i-\kappa}=\hat{\mathfrak{b}}\right]\right)_{j\in [i-\kappa, \ldots, i]}$, we know that for all ${j\in [i-\kappa, \ldots, i]}$, $\mathbb{E}_{\mathfrak{b}_j \sim \rho}\left[Q^*(\mathfrak{b}_j, a^*)|\mathfrak{b}_{i-\kappa}=\hat{\mathfrak{b}}\right] \geq \mathbb{E}_{\mathfrak{b}_j \sim \rho}\left[Q^*(\mathfrak{b}_j, a)|\mathfrak{b}_{i-\kappa}=\hat{\mathfrak{b}}\right]$.    It is straightforward to verify that for any $L$-smooth sequences of real numbers $(A_i)_{i \in [1,\ldots, k]} \geq (B_i)_{i \in [1,\ldots, k]}$ of length $\kappa$, such that $A_\kappa - B_\kappa > 2\kappa L$, the discounted sum $\sum_{i=1}^\kappa \gamma^{i-1}(A_i - B_i)$ must be greater than $T$. Hence, the cumulative expected regret of taking $a^*$ vs. $a$ following state $\hat{\mathfrak{b}}$ must exceed $T$ and $\sigma$ must have switched by time $i$, which is a contradiction with the assumption that $\sigma$ executes $a$.   

\end{proof}